\documentclass[10pt,twocolumn,letterpaper]{article}

\usepackage{cvpr}              % To produce the CAMERA-READY version
\usepackage{graphicx}
\usepackage{amsmath}
\usepackage{amssymb}
\usepackage{booktabs}
\usepackage{threeparttable}

\usepackage[resetlabels,labeled]{multibib}
\newcites{A}{Appendix References}
\usepackage{amsmath,amssymb,amsfonts}
\usepackage{algorithm2e}
\usepackage{xcolor}
\usepackage{subcaption}
\usepackage{hyperref}
\usepackage{bm}

\usepackage{afterpage}
\usepackage{times}
\usepackage{epsfig}
\usepackage{graphicx}
\usepackage{amsmath}
\usepackage{amssymb}
\usepackage[english]{babel}
\usepackage{amsthm}
\newcommand\blankpage{%
    \null
    \thispagestyle{empty}%
    \addtocounter{page}{-1}%
    \newpage}

\newtheorem{theorem}{Theorem}
\newtheorem{lemma}[theorem]{Lemma}
\usepackage[english]{babel}

\newtheorem{proposition}{Proposition}
\newtheorem{definition}{Definition}

\usepackage{wrapfig}
\usepackage[capitalize]{cleveref}
\crefname{section}{Sec.}{Secs.}
\Crefname{section}{Section}{Sections}
\Crefname{table}{Table}{Tables}
\crefname{table}{Tab.}{Tabs.}

\def\cvprPaperID{10433} % *** Enter the CVPR Paper ID here
\def\confName{CVPR}
\def\confYear{2022}

\begin{document}

%\title{Multidimensional Belief Quantification for Label-Efficient Uncertainty-Aware Meta-Learning}

\title{Multidimensional Belief Quantification for Label-Efficient Meta-Learning}

\author{Deep Pandey $\quad$ Qi Yu\\
Rochester Institute of Technology\\
{\tt\small \{dp7972,qi.yu\}@rit.edu}}
%\and
%Qi Yu\\
%Rochester Institute of Technology\\
%Rochester, NY, USA \\
%{\tt\small qyuvks@rit.edu}
%}
\maketitle

%%%%%%%%% ABSTRACT
\begin{abstract}
Optimization-based meta-learning offers a promising direction for few-shot learning that is essential for many real-world computer vision applications. However, learning from few samples introduces uncertainty, and quantifying model confidence for few-shot predictions is essential for many critical domains. Furthermore, few-shot tasks used in meta training are usually sampled randomly from a task distribution for an iterative model update, leading to high labeling costs and computational overhead in meta-training. We propose a novel uncertainty-aware task selection model for label efficient meta-learning. The proposed model formulates a multidimensional belief measure, which can quantify the known uncertainty and lower bound the unknown uncertainty of any given task. Our theoretical result establishes an important relationship between the conflicting belief and the incorrect belief. The theoretical result allows us to estimate the total uncertainty of a task, which provides a principled criterion for task selection. A novel multi-query task formulation is further developed to improve both the computational and labeling efficiency of meta-learning. Experiments conducted over multiple real-world few-shot image classification tasks demonstrate the effectiveness of the proposed model.
\end{abstract}

%%%%%%%%% BODY TEXT
\section{Introduction}\vspace{-2mm}
\label{sec:intro}
Deep learning (DL) models have achieved state-of-the-art performance for many computer vision applications. 
However, the effectiveness of DL models is challenged by some specialized domains (\eg medicine, biology, and security intelligence), in which labeled data for model training may be scarce. Unlike the DL models, human beings can learn efficiently from limited training samples by using the prior knowledge stored in their brains and applying it to new tasks. For example, once a child learns how to distinguish between lions and tigers, s/he can quickly generalize the concept to distinguish lions and cats with little or no additional training. Inspired by such human learning, various few-shot learning techniques ~\cite{finn2017model,vinyals2016matching,munkhdalai2017meta} have been developed, which provide a promising approach to address the label scarcity problem for DL models.

In recent successful few-shot learning approaches, the model is trained from multiple few-shot tasks comprised of few labeled examples instead of one large dataset as in the traditional setting. By learning from many similar tasks, the model can accumulate the shared knowledge among tasks. After training, it uses the knowledge gained from similar tasks as the prior knowledge to perform well on new unseen few-shot tasks. Meta-learning is one popular approach for few-shot learning where the model learns at two stages: {\em rapid learning} within a new task, which is guided by prior knowledge gained from {\em gradual learning} across tasks \cite{santoro2016meta}. In meta-learning, the model is trained on a large number of few-shot tasks to learn the shared inter-task knowledge. The learned model is evaluated based on its generalization capabilities on unseen few-shot tasks. 

Few-shot tasks have limited data to learn from (in some cases just 1 example/class). So, some model predictions may not be reliable. For critical applications (\eg autonomous driving), it is essential to quantify the prediction uncertainty. Some existing approaches indirectly provide uncertainty information of few-shot tasks by learning a posterior predictive distribution for testing data samples \cite{grant2018recasting,gordon2018metalearning,ravi2018amortized,finn2018probabilistic,kim2019edge}. However, they usually suffer from a high computational cost and rely on assumptions/approximations that may be invalid in practice.

Additionally, few-shot tasks used in meta training are usually sampled randomly from a task distribution formed using a large pool of labeled data samples.  Thus, meta-training for many optimization-based meta-learning approaches is computationally expensive, requiring evaluation of the second-order derivative (\ie, Hessian) of the (global) model parameters over each of the sampled tasks. Furthermore, the large number of tasks leads to high labeling costs in many real-world problems. However, not all the tasks contribute equally to the learning of the (global) model parameters, and evaluating the Hessian over these tasks can significantly slow down the meta training process. 

In this paper, we present a {\em novel \textbf{Un}certa\textbf{i}nty-aware \textbf{t}ask \textbf{s}election model for efficient meta-learning} (referred to as Units-ML) that 
provides uncertainty estimation to quantify the model confidence in few-shot predictions. Building upon the theory of subjective logic~\cite{josang2016subjective}, we formulate a multidimensional belief measure, including vacuous, conflicting, and incorrect beliefs, which can quantify both the {\em known uncertainty (KUN)} and {\em unknown uncertainty (UUN)} of a given task. However, evaluating incorrect belief relies on the labels of a query set in a task, making the UUN not accessible during task selection. We address this issue by
proving a novel relationship between conflicting belief and incorrect belief, which allows us to bridge the gap to UUN.
A novel task selection function is designed accordingly that integrates both KUN and UUN for  {\em belief-oriented} label-efficient meta-learning. 

We summarize our key contributions below: (1) a novel computationally and label-efficient meta-learning model that can estimate uncertainty in few-shot tasks, (2) a multidimensional belief measure to quantify the KUN and lower bound the UUN of a given task, (3) theoretical justification that conflicting belief lower bounds incorrect belief, which allows UUN estimation without label information, and (4) an uncertainty-aware task selection criterion and a novel multi-query task formulation to improve both computation and label efficiency of meta-learning. We conduct intensive experiments over multiple real-world image datasets to demonstrate the effectiveness of the proposed Units-ML model in terms of accurate uncertainty estimation, computationally effective task selection, and label-efficient learning under a limited annotation budget. 

\section{Related Work}\vspace{-2mm}
In meta-learning, the meta-model aims to learn (prior) knowledge shared by relevant tasks over multiple training episodes so that the model can perform well in new few-shot tasks. The prior knowledge can be learned through embedding functions and similarity metrics as in metric-based models~\cite{vinyals2016matching,snell2017prototypical,garcia2018few,sung2018learning, chen2020new}. The prior knowledge can also be captured by a deep neural network that maps a training dataset to parameters of the task-specific meta-model as in model-based meta-learning approaches \cite{mishra2018a, garnelo2018conditional, kimanp18,Gordon2020Convolutional}.

In optimization-based meta-learning \cite{Sachin2017}, task-specific model parameters are learned from the meta-dataset using an optimization procedure such that the model can adapt quickly with only a few examples from a new task. In Model-Agnostic Meta-Learning (MAML) \cite{finn2017model}, a good global initialization is learned from which the model can adapt quickly to new tasks using only a few data samples using a few gradient descent steps.
Some improvements for MAML include MetaSGD \cite{li2017meta} and MAML++ \cite{antoniou2018how} that help further improve the generalization and stability of MAML. First-order approaches such as Reptile\cite{nichol2018first} have also been developed to address high computational costs in meta-training of MAML. Meta-learning has recently been extended to the Bayesian setting \cite{grant2018recasting, gordon2018metalearning,finn2018probabilistic,yoon2018bayesian} to develop uncertainty awareness. We discuss these relevant uncertainty-aware meta-learning works in the Appendix.

Recent works have also attempted to show the effectiveness of task selection for meta-learning in reinforcement learning problems \cite{NEURIPS2020_ec3183a7,kaddour2020probabilistic}. In terms of task selection for few-shot classification problems, MTL \cite{sun2019meta} and GCP  \cite{liu2020adaptive} share similar motivations to our approach. In MTL, a two-stage hard-task scheme is introduced where the model is first trained on a batch of tasks and a list of failure classes (based on query set loss) is maintained. In the second stage, the model trains from hard tasks created using the failure classes that lead to better generalization. In GCP, a class-pair-based task sampling scheme is developed as an effective alternative to existing uniform sampling for meta-learning. In GCP, the class-pair potential matrix is used to sample the training tasks. GCP requires keeping track of pairwise potential among all training classes, and might not scale well when the number of training classes is large, or when new training classes are introduced as training progresses. These approaches can be applied complementary to our approach as they focus on determining the most informative classes (instead of tasks as in our model) from which the task is to be generated. With these approaches, once the candidate classes are determined, our model could be applied to formulate a multi-query task and select the most informative task for effective meta-learning.
%These existing task-selection methods may find limited applicability as they assume that the meta-learning model has access to the underlying task generation process \ie, the classes from which the tasks are created. In contrast, we propose a general task-selection method which does not require access to underlying task generation process.

Our method is an instance of an optimization-based meta-learning with uncertainty awareness, \ie, our model outputs the uncertainty estimates along with the predictions for few-shot tasks. In contrast to the probabilistic meta-learning approaches discussed above, our method does not add any significant computation overhead. Also, by further leveraging the predicted multidimensional belief (\ie, vacuous, conflicting, and incorrect), we perform belief-oriented task selection for uncertainty-ware meta-learning with faster and better convergence, augmented with multidimensional belief based uncertainty quantification.

\section{Methodology}
 \vspace{-2mm}
In this section, we present the proposed belief-oriented task selection for efficient meta-learning (Units-ML). We first describe the standard problem setup for few-shot learning. We then provide an analysis on the computational cost of MAML that motivates the need to perform belief-oriented task selection that sets the stage for us to describe the proposed Units-ML model in detail. 

\vspace{1mm}\noindent{\bf Problem setup.} We focus on few-shot classification problems and follow the episodic training procedure introduced in \cite{vinyals2016matching}. In particular, multiple tasks sampled from a task distribution are considered where each task consists of a support set $S$ and a query set $Q$. Specifically, a task in a $N$-way $K$-shot classification problem is defined as
\begin{align}
\nonumber \mathcal{T} &= \{ S, Q \}\\    
\nonumber S = \{X_S, Y_S\} &= \{(\bm{x}_1,l_{1}), ..., (\bm{x}_{N_s},l_{N_s})\}\\
Q = \{X_Q, Y_Q\} &= \{(\bm{x}_1,l_1),.....,(\bm{x}_{N_q},l_{N_q})\}
\end{align}
Support set $S$ has a total of $N_s = N \times K$ instances with $K$ examples/class, and query set $Q$ has $N_q$ new examples belonging to one of the $N$ classes. During meta-training, both support and query sets are used to train the model; during meta-testing, the model performs adaptation using the support set and is evaluated on the query set.

Besides the standard problem setup as above, we also consider the setting where only limited samples can be annotated due to a limited labeling budget and the goal is to train the meta-model in a label-efficient way. We assume each task consists of small support set with limited labeled samples along with an unlabeled query set with varied sizes. We want the meta-model to be trained such that it can perform well on any new samples of the task (\ie, any query set) after learning from the knowledge of the support set. 
% Also, as the query set determines the update to the global parameters of the model, it is essential to choose the most informative tasks to label their query sets. 
\paragraph{Analysis of MAML.}

MAML aims to learn a good initialization over multiple meta-iterations using the support-query setup discussed above. In each meta-iteration, a batch of tasks updates the model's global parameters. The updates in MAML can be summarized in two iterative steps: a local update using the support set and a global update using the query set. For each task, the local update proceeds as:
\begin{align}
   \nonumber \theta_0 &= \theta \ \ \textit{[make copy of global parameters]}\\
  \nonumber \theta_1 &= \theta_0 - \alpha \nabla_{\theta_0} \mathcal{L}\big[{f(\theta_0,X_S), Y_S}\big]\quad  \text{...}\\%\nonumber
    \theta_M &= \theta_{M-1} - \alpha \nabla_{\theta_{M-1}} \mathcal{L}\big[{f(\theta_{M-1},X_S), Y_S}\big] \label{taskAdaptation}
\end{align}
Here, the model $f$ outputs the predictions for support set input $X_S$ based on the parameters $\theta_m, m \in [1,M]$. The support set prediction $f(\theta_m, X_S)$ and support set ground truth $Y_S$ is used to compute the loss $\mathcal{L}$ and the $m^{th}$ 
 local update is done based on this loss. After $M$ local updates, the global parameters are updated using the query set input $X_Q$ and the query set ground truth $Y_Q$ as:
\begin{align}
    \theta_{\text{new}} &= \theta - \beta \nabla_\theta \mathcal{L}\big[{f(\theta_{M},X_Q), Y_Q } \big]
\end{align}
%\noindent{\bf Complexity analysis.} 
Denote the query set loss $\mathcal{L}\big[{f(\theta_{M},X_Q), Y_Q }\big]$ by $\mathcal{L}_{M}^Q$ and support set loss $L\big[{f(\theta_{M},X_S), Y_S }\big]$ by $\mathcal{L}_{M}^S$.
After $M$ local updates using the support set $S$, we update global parameters using the query set $Q$ as:
\begin{align}
% \begin{split}
  \nonumber  \theta_{\text{new}} &= \theta - \beta \nabla_{\theta} \mathcal{L}_M^Q = \theta - \beta \nabla_{\theta_M} \mathcal{L}_M^Q \times \nabla_\theta \big[\theta_M\big] \\
    &= \theta - \beta \nabla_{\theta_M} \mathcal{L}_M^Q \times \bigg(\prod_{m=M}^1 \nabla_{\theta_{m-1}} \big[\theta_{m}] \bigg)\times \nabla_\theta\big[\theta_0\big] \nonumber \\
    &= \theta - \beta\nabla_{\theta_M} \mathcal{L}_M^Q (I-H_{M-1})..\times (I-H_0) \times I
% \end{split}
\label{MAMLglobalUpdate}
\end{align}
where $\nabla_{\theta_M} \mathcal{L}_M^Q$ is a vector of length same as the number of parameters in $\theta$, $I$ is the identity matrix, and $H_{m} = \nabla_{\theta_{m}} \big[ \nabla_{\theta_{m}} \mathcal{L}_{m}^S\big]$ is the Hessian matrix. 
%It should be noted that global parameters are updated for each sampled task.
%However, {\em not all tasks contribute equally to learning good global parameters}. 
As shown above, the global parameter update is through the loss over the query set samples $\mathcal{L}_M^Q$, with $\theta_m$ implicitly capturing the support set information. To achieve label-efficient meta-learning, we need to quantify the informativeness of a task through its query set. Furthermore, global parameter update involves multiple Hessian-gradient products, which are computationally expensive. In standard meta-learning, a large number of tasks need to be labeled and then used for episodic training to find good global parameters. This not only incurs a high annotation cost but also takes a long time to converge. The proposed Units-ML model aims to {\em select the most informative tasks} for efficient meta-learning to reduce both the label and computational cost. 

\subsection{Multidimensional Task Belief Quantification}
% \subsection{Subjective Logic Based Uncertainty Prediction}
% The proposed uncertainty-aware task selection model chooses tasks based on their predicted uncertainty. 
%and obtain predicted uncertainty 
%We consider Subjective Logic (SL) to develop uncertainty-awareness in our meta-learning model. 
% \paragraph{Subjective Logic (SL).}
We formulate a novel multidimensional belief-based measure to quantify different types of task uncertainty in meta-learning by leveraging the formalism from Subjective Logic (SL)~\cite{josang2016subjective}. 
SL considers $N$ opinions (corresponding to $N$ classes) and assigns belief masses to each opinion ($b_1, b_2, ... b_N$) along with an overall uncertainty mass $u$. The belief masses represent the total evidence from the model whereas the uncertainty mass represents the vacuity (i.e., {\em lack of evidence}), and the two masses sum to 1: 
\begin{align}
\sum_{n=1}^N b_n + u = 1, \ \forall{n}: \ 0 \leq b_n \leq 1, \ 0 \leq u \leq 1    
\end{align}
By explicitly considering the uncertainty mass and using evidence-based measure (vacuity) to quantify it, we can obtain the model's {\em vacuous belief} on given tasks. By learning tasks with a high vacuity, the model can gain the lacking knowledge. Furthermore, we can also capture the uncertainty due to the {\em conflicting belief} using the dissonance ($dis$) that is complementary to the vacuity ($u$):% due to the model's missing knowledge:
\begin{align}
\label{eq:dis}
    &dis = \sum_{n=1}^N \Big( b_n \frac{ \sum_{j\neq n} b_j Bal(b_j, b_n)}{\sum_{j\neq n} b_j } \Big),\\ 
    &Bal(b_j, b_n) = 
\begin{cases}
  1 - \frac{|b_j - b_n|}{b_j+ b_n}, & \text{if } b_i b_j > 0\\
  0, & \text{otherwise}
\end{cases}
\end{align}
where $Bal(\cdot, \cdot)$ is the relative mass balance function between two belief masses. 
By learning tasks with a high dissonance, the model can correct its acquired conflicting knowledge to ensure more accurate predictions. Additional discussions about SL is presented in the Appendix.

The theory of SL can be conveniently embedded in a standard (non-Bayesian) neural network, making it computationally attractive. In particular, a neural network can form multinomial opinions in classification by replacing the final softmax layer with a non-negative activation layer~\cite{sensoy2018evidential}. As a result, the network is trained to predict an evidence vector ${\bm e}$ = $(e_{1}, e_{2},...e_{N})$ for a given input ${\bm x}$. 
The belief and vacuity are then computed as 
\begin{align}
\label{eq:vac}
b_n = \frac{e_n}{S}, \quad u = \frac{N}{S}, \quad \text{where~} S = \sum_{n=1}^{N}(e_{n} + 1),
\end{align} 
By setting $\alpha_n=e_{n} + 1$, the probability of assigning ${\bm x}$ to the $n$-th class is $\frac{\alpha_n}{S}$. If we use a set of categorical random variables $(p_1,...,p_N)^\top$ to represent the class assignment probabilities, then $\alpha$'s are essentially the concentration parameters of a Dirichlet prior $\text{Dir}[(p_1,...,p_N)^\top|(\alpha_1,...,\alpha_N)^\top]$.   

% \subsection{Multidimensional Task Belief Quantification}
\paragraph{Multidimensional Belief for Task Uncertainty Quantification.} 
To quantify the vacuous belief (\ie, vacuity) and conflicting belief (\ie, dissonance) for a given task $t$ that consists of a support set with limited labeled instances along with an unlabeled query set $q$, we propose to perform meta-testing on each of the sampled tasks. In particular, the model first adapts to the task by using the support set $S$. Then, task vacuity and dissonance are evaluated using the unlabeled query set $Q$. We compute the vacuity and dissonance using \eqref{eq:vac} and \eqref{eq:dis} for each data sample in the query set of a task. The {\em vacuous belief} $vb^t$ and {\em conflicting belief } $cb^t$ of task $t$ are computed as the average of vacuity and dissonance of query set samples: % and belief masses , and {\em task belief } ${\bm b}^t$
\begin{align}
\label{eq:tv}    vb^t =& \frac{1}{N_q} \sum_{q=1}^{N_q} u_{q}^t & \text{\bf (Vacuous Belief)}\\
% \label{eq:tb}    {\bm b}^t =& \frac{1}{N_q} \sum_{q=1}^{N_q} {\bm b}^t_{q} \quad \text{\bf (Task Belief)}\\
\label{eq:tdis}    cb^t =& \frac{1}{N_q} \sum_{q=1}^{N_q} {dis}^t_{q} & \text{\bf (Conflicting Belief)}%\\
%\nonumber ib^t =& \frac{1}{N_q} \sum_{q=1}^{N_q} ||\bm{b}_{q}^t \odot (\bm{1} - \bm{y}_{q}^t)||_1 & \text{(\bf{Incorrect Belief})}\\
%\nonumber\label{eq:tdis}    tb^t =& \frac{1}{N_q} \sum_{q=1}^{N_q} ||\bm{b}_{q}^t \odot \bm{y}_{q}^t||_1 & \text{(\bf{True Belief})}
\end{align}
where $u_{q}^t$ and $dis_{q}^t$ are the vacuity and the dissonance for the $q^{th}$ query sample in task $t$. Since vacuity reflects the model's lack of evidence on a data sample, vacuous belief indicates the model's overall lack of knowledge on a task. 
Therefore, selecting a task with a high $vb^t$ and performing meta-training using \eqref{MAMLglobalUpdate} can adjust the global parameters to effectively learn the missing knowledge. As a result, the model is expected to perform well on similar unseen few-shot tasks in the future. While vacuous belief captures one source of uncertainty that indicates the model's lack of knowledge for the task, conflicting belief helps identify the difficult tasks, \ie, the tasks in which the model gets confused between different classes.
%There is another source of uncertainty referred to as dissonance which captures the challenging data instances for which the model gets confused between different classes. 
Learning from these tasks can help adjust the global parameters such that the model can correct the acquired confusing knowledge. As a result, the model is expected to be able to better differentiate different classes within a task. 

Since both vacuous belief and conflicting belief can be quantified without knowing the labels of the query set, they are instances of {\em known uncertainty (KUN)}. There is another source of uncertainty, referred to as {\em unknown uncertainty (UUN)}, that the model is unaware of. UUN usually leads to a highly confident wrong prediction that can cause more severe consequences in critical domains (\eg, autonomous driving). This type of uncertainty is essentially caused by model overfitting, which can be quite common when applying deep learning models to few-shot problems. It is critical to train the meta-learning model to minimize the unknown uncertainty so that the model can avoid making over-fitted predictions in the future. UUN can be captured by a third type of belief, referred to as {\em incorrect belief:} 
\begin{align} \label{eq:incorrectBelief}
ib^t = \frac{1}{N_q} \sum_{q=1}^{N_q} ||\bm{b}_{q}^t \odot (\bm{1} - \bm{y}_{q}^t)||_1 \; \;\text{(\bf{Incorrect Belief})}
\end{align}
where $\bm{y}_{q}^t =  (y_{q,1}^t, ... y_{q,N}^t)^T$ is the one-hot vector representing the ground-truth label of the $q^{th}$ query set sample, ${\bm b}_{q}^t =(b_{q,1}^t, ... b_{q,N}^t)^T$ is the $N$-dimensional belief vector,  $\odot$ represents element-wise multiplication, and $||\cdot||_1$ is the $l_1$ norm. Intuitively, when the model is wrongly confident, \ie, the model places a strong belief in a class that is different from the true class label, it will contribute a large incorrect belief component. The task-level incorrect belief aggregates these components to reflect the overall UUN of the task. 

However, a key limitation of incorrect belief is that computing incorrect belief requires the query set labels, making it infeasible to be used in task selection for label-efficient meta-learning. We address this issue through an important theoretical result as presented in the following theorem. The theorem establishes an important relationship between incorrect belief and conflicting belief, which essentially bridges the gap between unknown uncertainty and known uncertainty. 

\begin{theorem}[\textbf{Lower bound of incorrect belief}]
\label{pythagorean}
Consider an unlabeled task $t$ with (unknown) incorrect belief of $ib^{t}$ and conflicting belief $cb^{t}$. Then, incorrect belief is lower bounded by a half of the conflicting belief on the same task
%the dissonance of the task always lower-bounds the incorrect belief to a scaling factor of $2$ 
\begin{align}
\label{eq:disIncBelreln}
     ib^{t} \ge \frac{1}{2} cb^t \quad \text{where}\quad 0 \leq cb^{t} \leq 1,  0 \leq ib^{t} \leq 1
\end{align}
\end{theorem}
\begin{proof}[Proof sketch]
We first consider a sample within the task for which the model outputs an $N$-dimensional belief vector. We consider the analytical expression of the conflicting belief, simplify the relative mass balance between the beliefs, and expand the different belief terms of the conflicting belief. After expanding and rearranging, we find an upper bound for each of the terms in the conflicting belief expression that proves that for any sample, the incorrect belief is lower bounded by half the conflicting belief. Finally, we generalize the relationship to be true for any task. 
\end{proof}
Due to space limitations, the complete proof of the theorem is provided in the Appendix. Ideally, we want to select tasks with high incorrect belief that encourages the model to correct the model's incorrect knowledge. Since the conflicting belief provides a lower bound for incorrect belief, it provides a way to estimate the incorrect belief (and reduce UUN) without the label information, which is instrumental for active task selection.

\begin{figure*}[t!]
        \includegraphics[width=0.90\textwidth]{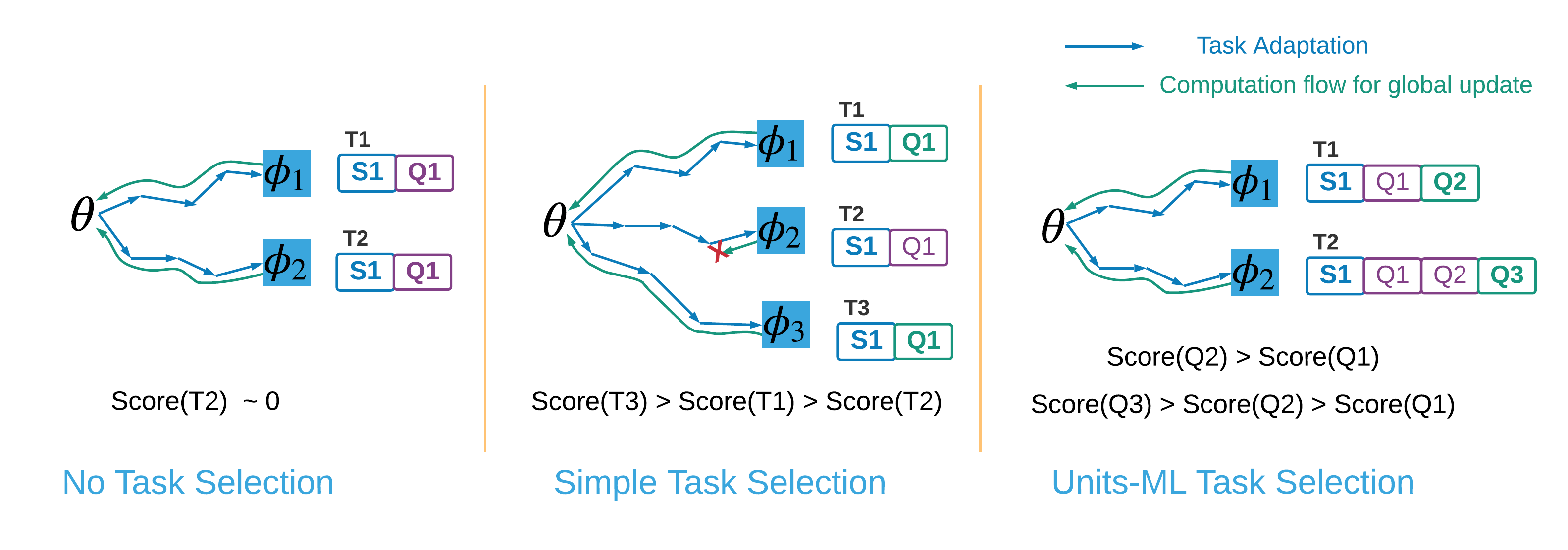}%old/results/ts2.pdf}
        \centering
        \vspace{-4mm}
        \caption{The left figure shows the idea of optimization-based meta-learning, the middle figure shows the idea of task selection (Units-ST), and the right figure illustrates the idea of our proposed task selection method with multi-query tasks. The forward arrows indicate the task adaptation \eqref{taskAdaptation}, and the backward arrows represent the computations for global parameter update  \eqref{MAMLglobalUpdate}.
        }
        \label{fig:repIdeaFig}
        \vspace{-2mm}
\end{figure*}
\subsection{Belief-Oriented Task Selection and Training}
The multidimensional belief provides a principled way to quantify both the KUN and UUN of different tasks without the query set labels. With this, the most informative tasks are the ones with the greatest overall uncertainty, including both KUN and UUN. For the former, it is captured by two different types of belief: vacuous and conflicting. As for the latter, we can obtain its lower bound through conflicting belief. Since conflicting belief is used to quantify both KUN and UUN, we propose a task selection function that integrates vacuous beliefs (vacuity) and conflicting beliefs (dissonance) to estimate the total task uncertainty: 
\begin{align}
\label{eq:score}
    \text{unc}^t = \lambda (vb^t) + (1 - \lambda) (cb^t) \ \ \text{(\bf{Task Uncertainty})}
\end{align}
where $\lambda$ is a balancing term that determines the relative importance between these two types of belief. 
Intuitively, the tasks with high vacuous belief represent new/unseen tasks on which the model is not able to make confident predictions, whereas the tasks with high conflicting belief represent challenging tasks on which the model struggles to confidently discriminate among the classes. We start with a relatively large $\lambda$ in the early phase of meta-learning so that the model can better explore the task space to fill out the knowledge gap. Then, the focus will shift to the conflicting belief to fine-tune the model on the more difficult tasks or tasks that the model has an incorrect knowledge. 
%Our model estimates the task informativeness using $\text{score}^t$ where the model looks at the model's lack of knowledge (vacuous belief) and the conflicting belief for the candidate tasks to decide the task whose query set is to be labeled.
% Intuitively, we start with a relatively large $\lambda$ in the early phase of meta-learning so that the model can better explore the task space to fill out the knowledge gap. Then, the focus will shift to the conflicting belief to fine-tune the model on the more difficult tasks or tasks that the model has an incorrect knowledge. 
%of task vacuity compared to the task dissonance in determining the overall informativeness of a task. Intuitively, the model first learns from the limited information of the support set. Afterwards, the model estimates the task informativeness using $\text{score}^t$ where the model looks at the model's lack of knowledge (task vacuity) and the models conflicting evidence (task dissonance) for all the tasks to decide the task whose query set is to be labeled. 

\paragraph{Multi-Query Tasks.} With the above selection score ($\text{unc}^t)$, we propose to conduct uncertainty-aware task selection with a novel task formulation strategy for label-efficient meta-learning. A straightforward method for task selection is to sample a large number of tasks (say $J$ tasks) from a task distribution  $p(\mathcal{T})$ and use task selection criteria to select $I$ tasks to be labeled and perform meta-learning. We refer to this strategy as Units-ST (see Figure~\ref{fig:repIdeaFig}). 
In Units-ST, for each discarded task, the model needs to adapt to the support set to determine the informativeness which is a waste of computation and support set labels. To further improve efficiency, we propose to formulate {\em multi-query} tasks, where each task consists of a shared support set and multiple query sets. In this new formulation, referred to as Units-ML (see Figure~\ref{fig:repIdeaFig}), the model will adapt to the support set in the task and choose the most informative query set to label. Other unlabeled query sets will be discarded and not used for meta-learning. Such Multi-Query Tasks can be an ideal choice for limited budget real-world few-shot problems. 
%For instance, consider a few-shot recommendation model that needs to estimate the user's rating for an item based on user's interactions to decide the item to recommend. In such scenario, the model can consider multiple items to make a multi-query task, and based on the uncertainty score $unc^{t}$, the model can decide the item to recommend, collect the labels, and meta-train on such tasks in a label-efficient manner. 

\paragraph{Belief Regularized Model Training.}
We aim to train the meta-model to learn a good initialization such that for a new task, after learning from limited data of the support set, the meta-model can make a prediction as well as output the confidence in the prediction (the uncertainty information). To this end, we assume that the label for each sample is obtained from a generative process with a Dirichlet prior and a multinomial likelihood as specified through the SL framework. The parameters for the Dirichlet prior express the vacuity and belief masses for uncertainty estimation. Further, we leverage the conjugacy between the Dirichlet prior and the multinomial likelihood. With this, we can learn these parameters by minimizing a loss between the multinomial output and the ground truth labels.

Additionally, while the incorrect belief can only be estimated through its lower bound during the task selection phase, once the task is selected, the labels of its query set will be collected. Consequently, the incorrect belief can be accurately quantified, which can be used to guide the model training (to minimize the incorrect belief). To this end, we propose a belief regularized loss function
\begin{align}
    \label{eq:theOverallLossMain}
    \mathcal{L}_i &= - \ln \int \text{Mult}({\bm y}_i|{\bm p}_i) \text{Dir}({\bm p}_i|\boldsymbol{\alpha}_i) \text{d} {\bm p}_i+ \eta {R}_{ib}\\
    {R}_{ib} &=  \mathbf{b}_i \odot (\mathbf{1}-\mathbf{y}_i)
\end{align}
where $\mathcal{L}_i$ is the loss on the $i$-th data sample with one hot label $\mathbf{y}_i$, ${R}_{ib}$ is the incorrect belief regularization for the sample, and $\eta$ is a regularization coefficient that balances between minimizing the incorrect belief and maximizing the log likelihood. Moreover, the model outputs $N-$dimensional evidence $\mathbf{e}_i$ from which the belief $\mathbf{b}_i$ and dirichlet parameters $\boldsymbol{\alpha}_i$ are obtained. Limited by space, we present descriptions of SL and the additional details about the incorrect belief regularization, loss function, design choices, and hyperparameter settings in the Appendix.

\section{Experiments}\vspace{-2mm}
We first carry out experiments to show the accurate multidimensional belief quantification, which empirically validates our theoretical results. We then conduct intensive experiments on real-world few-shot image classification tasks to demonstrate the effectiveness of the Units-ML model on (1) accurate uncertainty estimation for few-shot learning, (2) fast convergence with limited label budget, and (3) competitive meta-learning performance in terms of generalization of the learned model and flexibility of adjusting prediction results based on model confidence. To demonstrate the general applicability of the proposed model, we extend the MetaSGD models to make them uncertainty-aware and also conduct additional experiments on any-shot classification using \textit{mini}-ImageNet/CifarFS and in multi-dataset setting on Meta-Dataset \cite{DBLP:conf/iclr/TriantafillouZD20} as proposed by Bayesian TAML \cite{lee2020learning}. Limited by space, we report these results along with a detailed experiment setup in the Appendix.

{\bf Datasets.} We evaluate our proposed method on three real-world benchmark image datasets: Omniglot \cite{lake2015human}, \textit{mini}-ImageNet \cite{vinyals2016matching}, and CifarFS \cite{bertinetto2018metalearning}. Details of the datasets are summarized in Table \ref{tab:datasetDetails} of the Appendix. 
% \subsection{Dataset Details}
\subsection{Details of Comparison Baselines}

Our comparison includes optimization-based meta-learning models (MAML~\cite{finn2017model},
    MUMOMAML~\cite{vuorio2019multimodal},  
    CAVIA~\cite{zintgraf2018cavia},
    MetaSGD~\cite{li2017meta},
    Reptile~\cite{nichol2018first},  HSML~\cite{yao2019hierarchically}
    ) and models with uncertainty quantification capabilities (PLATIPUS \cite{finn2018probabilistic}, VERSA \cite{gordon2018metalearning}, BMAML \cite{yoon2018bayesian}, LLAMA \cite{grant2018recasting}, and ABML \cite{ravi2018amortized}).
    
A description of each comparison baseline is provided in the Appendix.
As some of these models do not release their source codes, we refer to the existing literature to report their performance. Thus, the results may not be available for all three datasets for some models. 

\vspace{1mm}\noindent{\bf Experiment setup.} We experiment with few-shot classification problems, where we consider $N$-way $K$-shot tasks with $q$ instances/class in the query set. In such a setting, tasks are created by randomly sampling $N$ classes and then sampling $K+q$ instances from each class. The $N\times K$ instances make the support set of the task, and there are $N\times q$ instances in the query set. 
The models are trained using Adam optimizer with the outer loop learning rate of 0.001 and evaluated on 600 validation set tasks. We train the model for 100 epochs where each epoch consists of 500 meta-iterations, and average the final test set performance across 3 independent runs.   In each meta-iteration, the model is trained with 8 tasks for Omniglot, 4 tasks for CifarFS \cite{bertinetto2018metalearning}, 4 tasks for 5-way 1-shot \textit{mini}-ImageNet, and 2 tasks for 5-way 5-shot \textit{mini}-ImageNet. We use the same standard 4-module convolutional architecture similar to ALFA \cite{NEURIPS2020_ee89223a}, Antoniou et al. \cite{antoniou2018how}. For limited labeling budget experiments, we train the models for 50 epochs with 100 meta-iterations where each meta-iteration consists of 2 tasks/batch (total of 10,000 tasks) for all the models. In the multi-query tasks, each task has 8 unlabeled query sets sharing the same support set. Additional implementation details are provided in the Appendix.

\subsection{Multidimensional Belief Quantification}

\begin{figure}[htpb]
\vspace{-3mm}
\centering
\begin{subfigure}{0.23\textwidth}
  \centering
  \includegraphics[width=\linewidth]{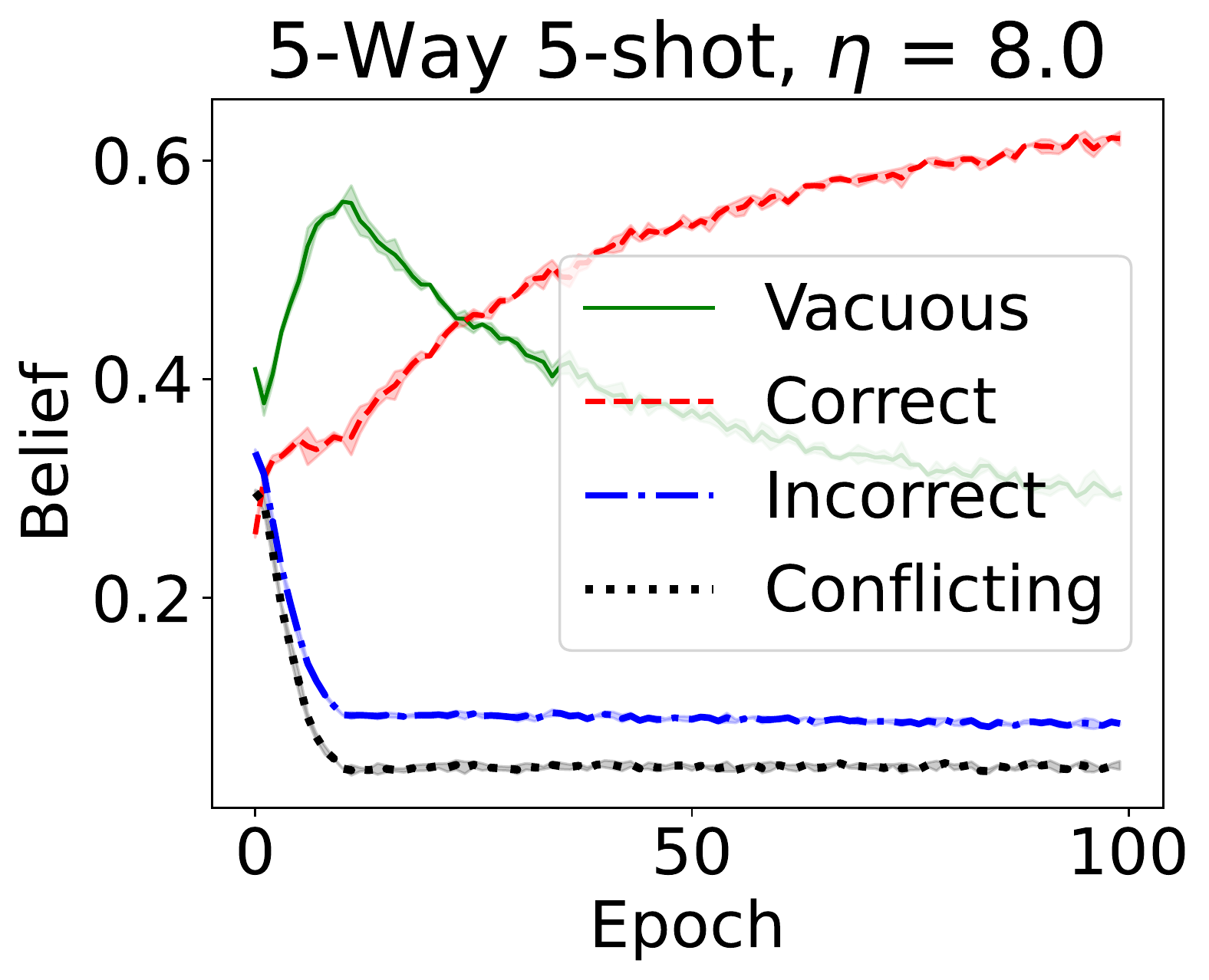}
  \caption{Training Trend}
\end{subfigure}
\begin{subfigure}{0.23\textwidth}
  \centering
  \includegraphics[width=\linewidth]{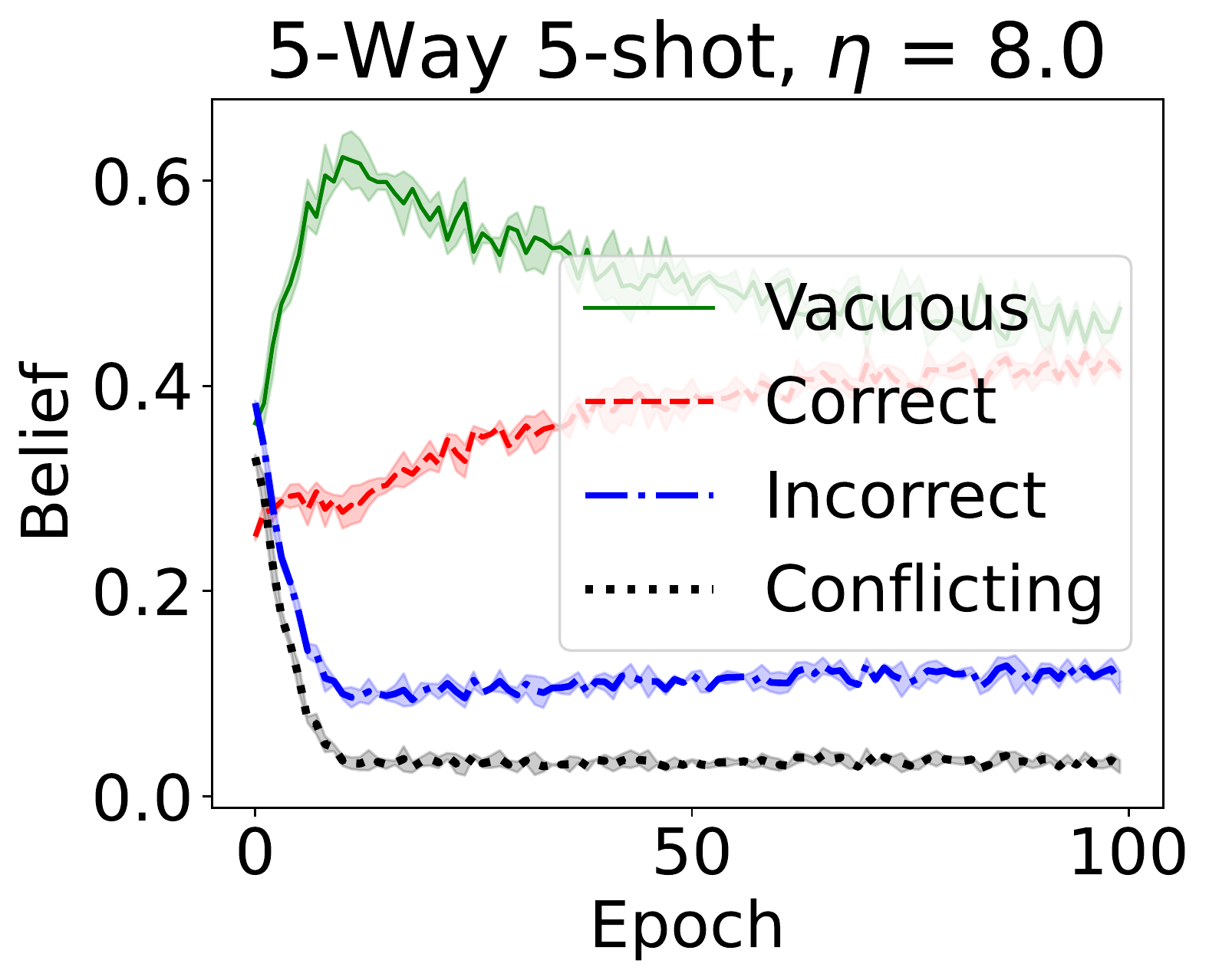}
  \caption{Validation Trend}
\end{subfigure} 
\caption{Multidimensional belief trend for CifarFS dataset
}
\label{fig:multibelief}
\vspace{-2mm}
\end{figure} 
We present the training and validation trends of the multidimensional belief for 5-way 5-shot CifarFS dataset in Figure~\ref{fig:multibelief}. There is low average vacuous belief and high incorrect belief at the initial training phase. This is most likely due to the model's overfitting on the limited training data, which could also indicate the importance of applying the proposed multidimensional belief to quantify both KUN and UUN under a limited labeling budget when overfitting is more likely to occur. In the early phase, since the model knows less, it also under-estimates the vacuous belief. In the immediate next few epochs, the model starts to make an accurate adjustment to its multidimensional beliefs as it is exposed to more samples. This leads to an increase in the correct belief and a decrease in all other beliefs as desired. It should be noted that the conflicting belief closely trails the incorrect belief in both the training and test tasks in all training phases. Thus, the figures empirically validate our theorem that the conflicting belief lower bounds twice the incorrect belief.

\subsection{Uncertainty Estimation Performance}
We then carry out experiments to assess the effectiveness of the proposed Units-ML model for uncertainty estimation in few-shot learning. For few-shot samples that remain new to the meta-model after being adapted to the samples in the support set, the model is expected to report a high vacuity; otherwise, the predicted vacuity should be low, reflecting high model confidence. Figure~\ref{repFig1} shows the predicted vacuity for a query set in a 5-way 1-shot Omniglot test task. As the query set images are rotated (by angles from 0$^{\circ}$ to 90$^{\circ}$ as indicated by R), the model starts to make mistakes (indicated by red) and becomes more uncertain in its predictions (as indicated by vacuity). Furthermore, when tested using MNIST characters as out-of-distribution (OOD) samples, the model accurately outputs a large vacuity, which shows its potential for OOD detection in few-shot learning. Figure~\ref{oodopensetmini} shows our model's performance on a 5-way 4-shot ImageNet test task with 3 in-distribution images, 2 open-set/OOD images (a cake image from a different \textit{mini}-ImageNet class, and a bird image from the CUB dataset \cite{WelinderEtal2010}) in the query set. For in-distribution samples, the model prediction is correct, and the prediction confidence (indicated by the vacuity and dissonance) is reasonable. For the OOD/open-set images, the model outputs high vacuity showing our model's potential in OOD/open-set detection. Furthermore, all the belief mass for OOD samples contributes to the incorrect belief (\ie, there is no correct belief for OOD samples). Our model outputs low dissonance for confidently correct predictions and comparably higher conflicting belief (the dissonance) for confusing samples, a highly desirable characteristic of a model with accurate uncertainty awareness. Additional illustrative examples and comparisons demonstrating our model's potential in open-set/OOD detection, and the model training process is provided in the Appendix.

\begin{figure}[t!]
        \includegraphics[width=0.45\textwidth]{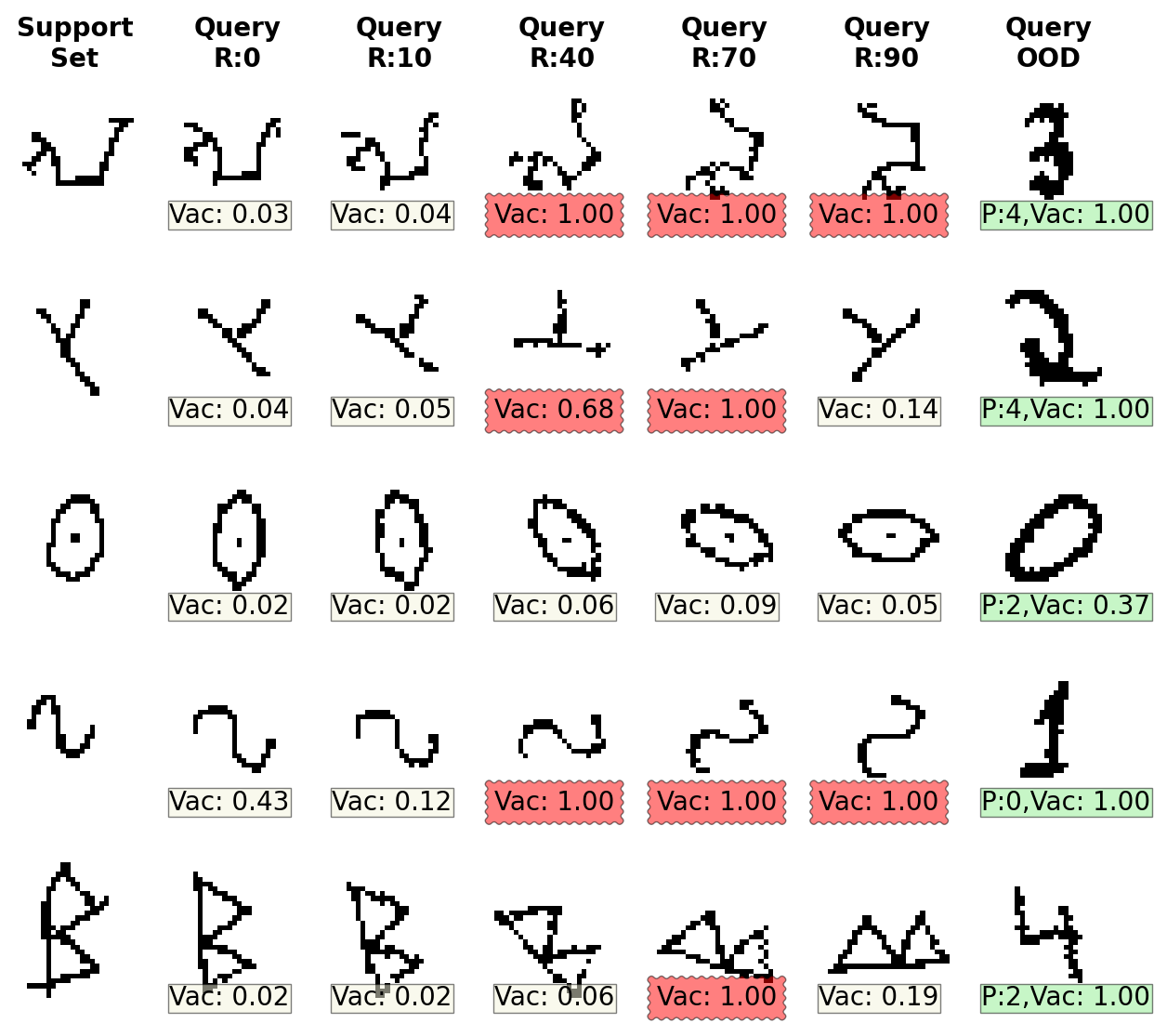}
        \centering
        \caption{Uncertainty prediction in a 5-w 1-s task
        %Vacuity values and predictions for query set in a 5-way 1-shot Omniglot test task with 1 instance per class in the query set. As the query set images are rotated (by angles from 0$^{\circ}$ to 90$^{\circ}$ as indicated by R), the model starts to make mistakes (indicated by red) and also becomes more uncertain in its predictions (as indicated by vacuity). Also, the model outputs large vacuity for OOD samples from MNIST \shortcite{lecun-mnisthandwrittendigit-2010} dataset as expected.
        }
        \label{repFig1}
        \vspace{-4mm}
\end{figure}
\begin{figure}[t!]
        \includegraphics[width=0.48\textwidth]{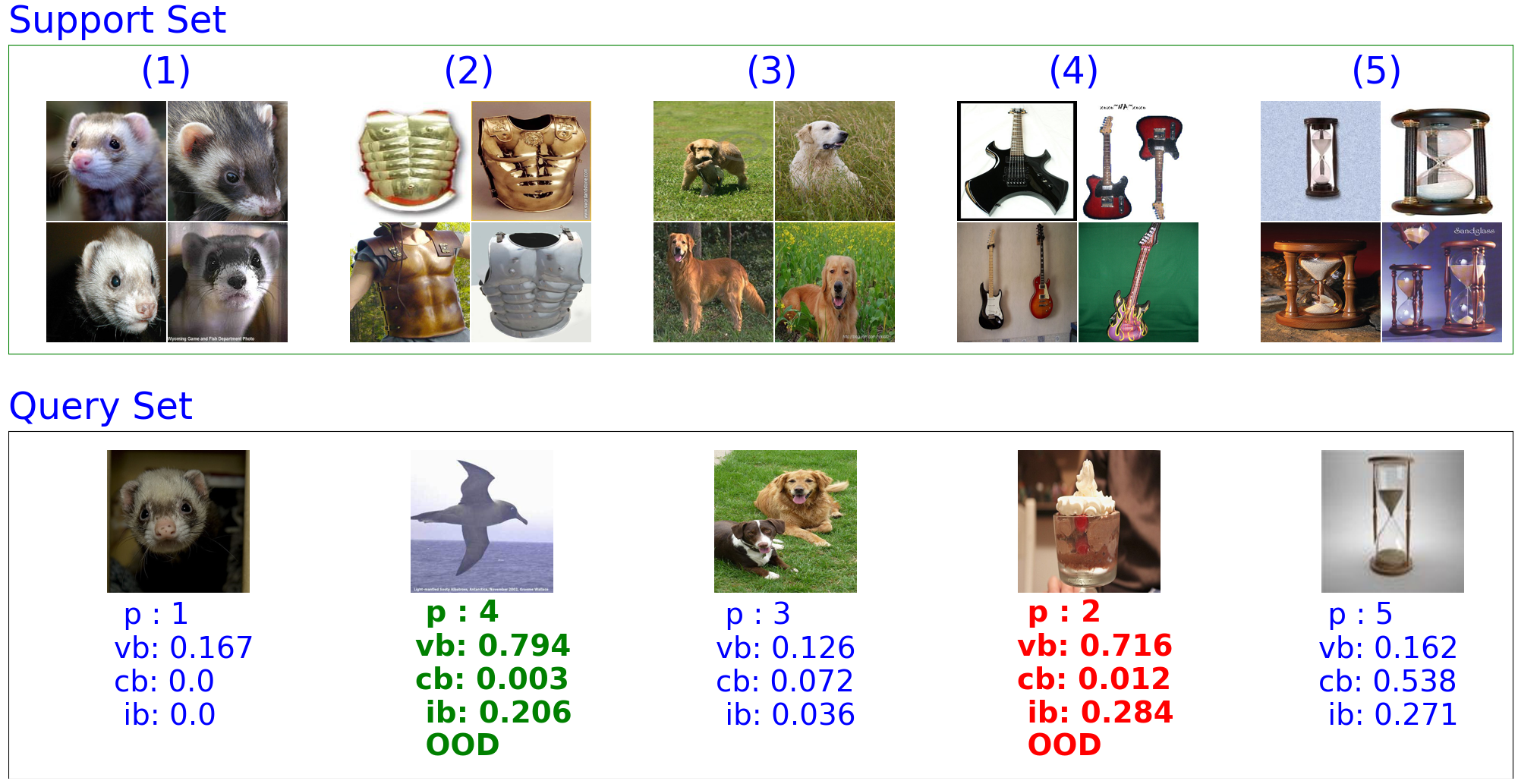}
        \centering
        \caption{Uncertainty prediction in a 5-w 4-s \textit{mini}-ImageNet test task with ood/open-set instances in query set. vb, cb, and ib represent the vacuous belief, conflicting belief, and incorrect belief, respectively for the prediction p.
        }
        \label{oodopensetmini}
\end{figure}
\begin{figure}[htpb] % VSPACE is not allowed
\vspace{-2mm}
\centering
\begin{subfigure}{0.23\textwidth}
  \centering
    \includegraphics[width=0.9\linewidth]{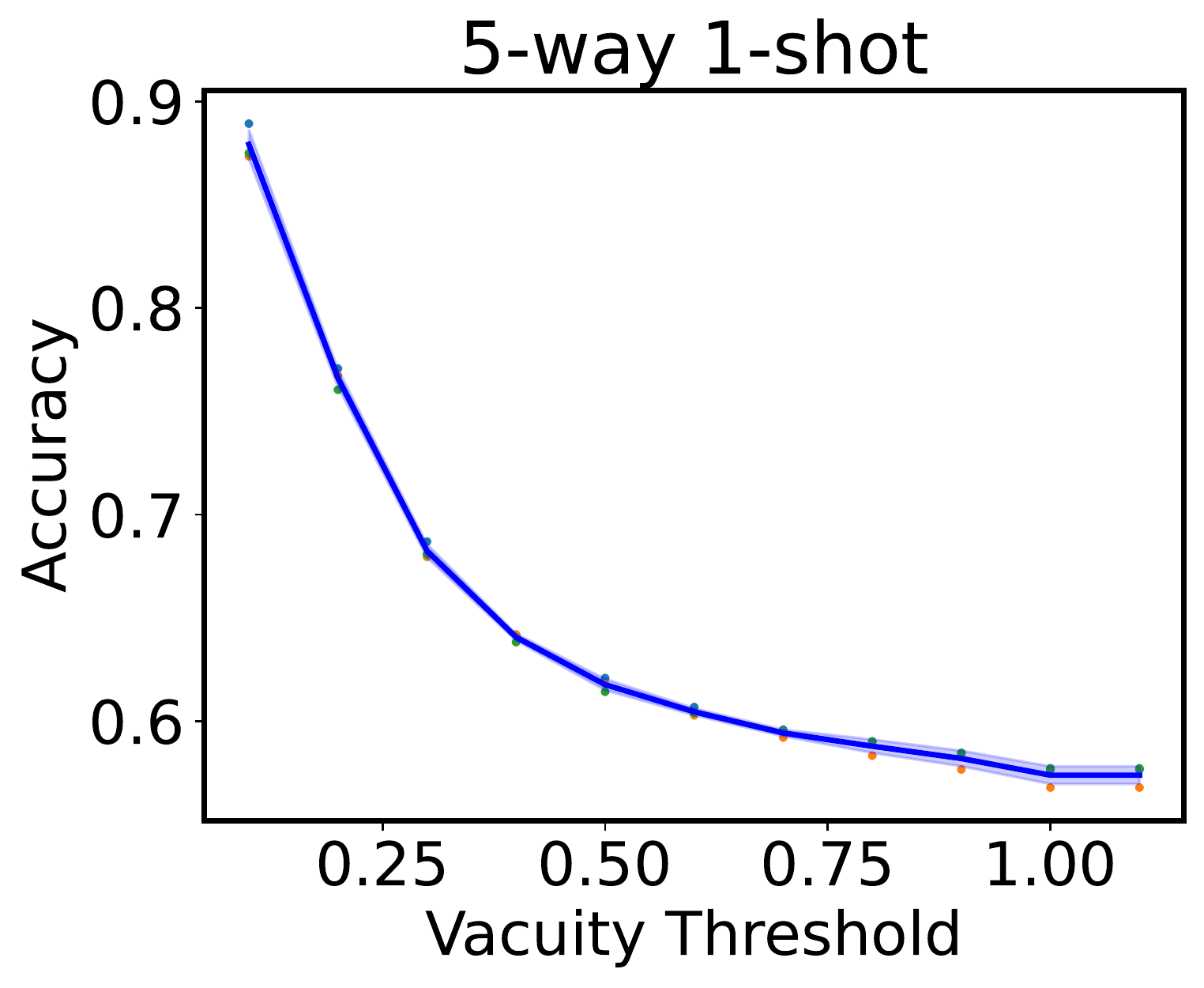}
  \caption{5-w 1-s CifarFS}
\end{subfigure}
\begin{subfigure}{0.23\textwidth}
  \centering
  \includegraphics[width=0.9\linewidth]{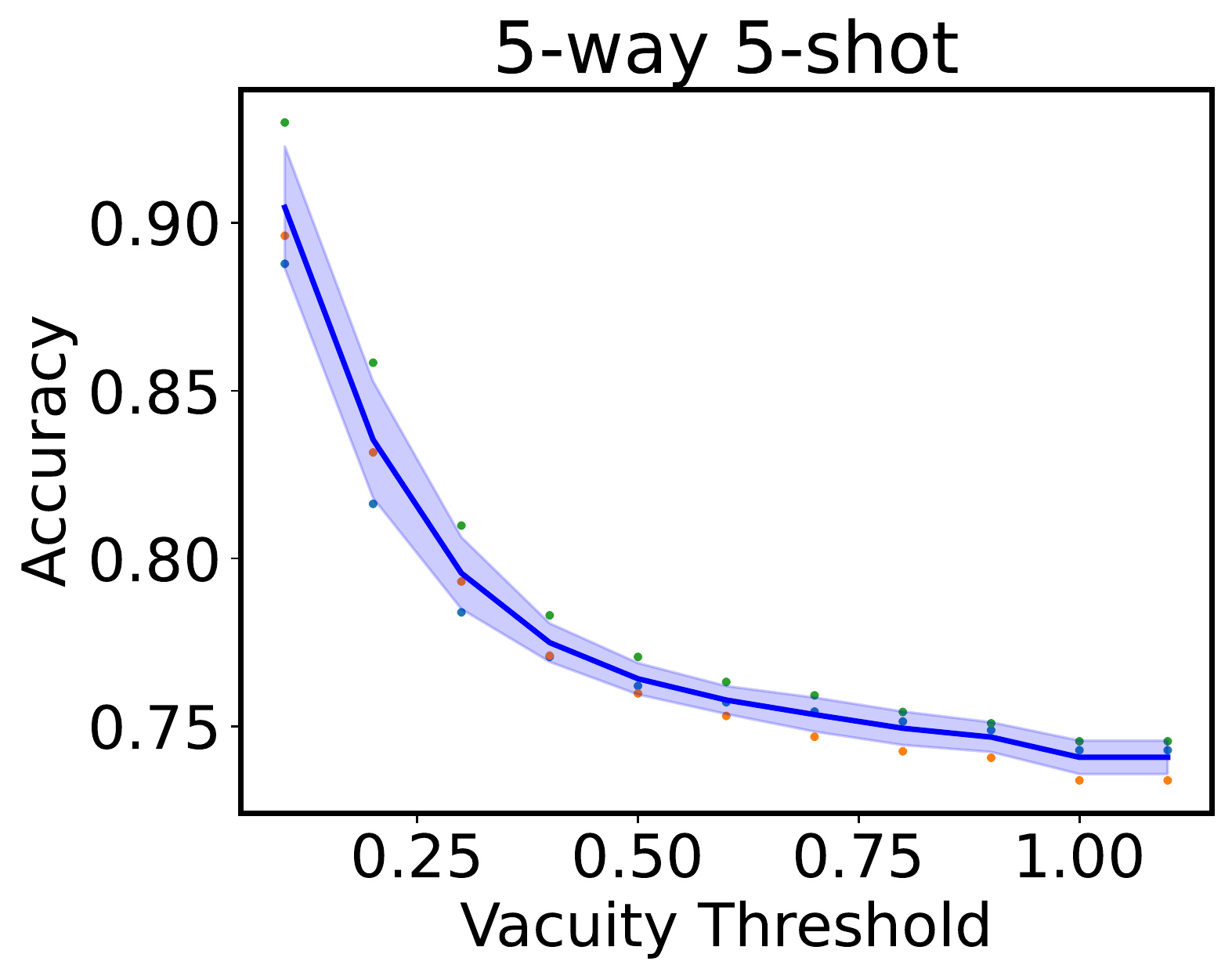}
  \caption{5-w 5-s CifarFS}%\textit{mini}-ImageNet} %TO DO - Update Figure
\end{subfigure}
\vspace{-2mm}
\caption{Prediction accuracy at different vacuity thresholds
}
\label{fig:advantageOfSm2}
\vspace{-2mm}
\end{figure} 
We further study the relationship between vacuity and the prediction accuracy for query set predictions to access the reliability of the uncertainty. Figure~\ref{fig:advantageOfSm2} visualizes how the prediction accuracy varies with vacuity using 5-way 1-shot and 5-way 5-shot tasks from CifarFS. For instance, in 5-way 5-shot CifarFS, by setting the vacuity threshold to 0.2 (considering prediction accuracy for samples whose vacuity is less than $0.2$), the model's prediction accuracy reaches around 85\%, which is around 10\% better than making predictions without consulting vacuity. The results for other datasets and settings are presented in the appendix that show a similar trend. Such flexibility can effectively avoid making less reliable predictions, a highly desirable property to facilitate decision-making in critical domains. 
\subsection{Active Task Selection}
Next, to demonstrate the label efficiency of our proposed model, we experiment under a limited labeling budget scenario. We consider that the model has access to the small labeled support set and the model needs to decide the tasks to be labeled (limited labeled budget). We consider a labeling budget of 10,000 tasks and compare our proposed belief-oriented task selection with the uncertainty aware Versa model that uses the model's predicted uncertainty on the query set for task selection and MAML model (which is not uncertainty aware) that randomly selects the query set of the task to be labeled. Our model uses a novel task uncertainty score in \eqref{eq:score} to select the tasks with the greatest task-level uncertainty to be labeled. Figure~\ref{fig:activetaskSelection} shows the results where we outperform the baselines in the early phase of the meta-learning process under limited task budget scenario. For instance, in $20$-way $5$-shot omniglot experiments, our model converges to more than 90\% accuracy only after 1000 iterations whereas the baseline models take significantly longer to converge to similar levels. Additional results are presented in the Appendix.
\begin{figure}[ht!] % VSPACE is not allowed
\centering
\begin{subfigure}{0.23\textwidth}
  \centering
  \includegraphics[width=0.9\linewidth]{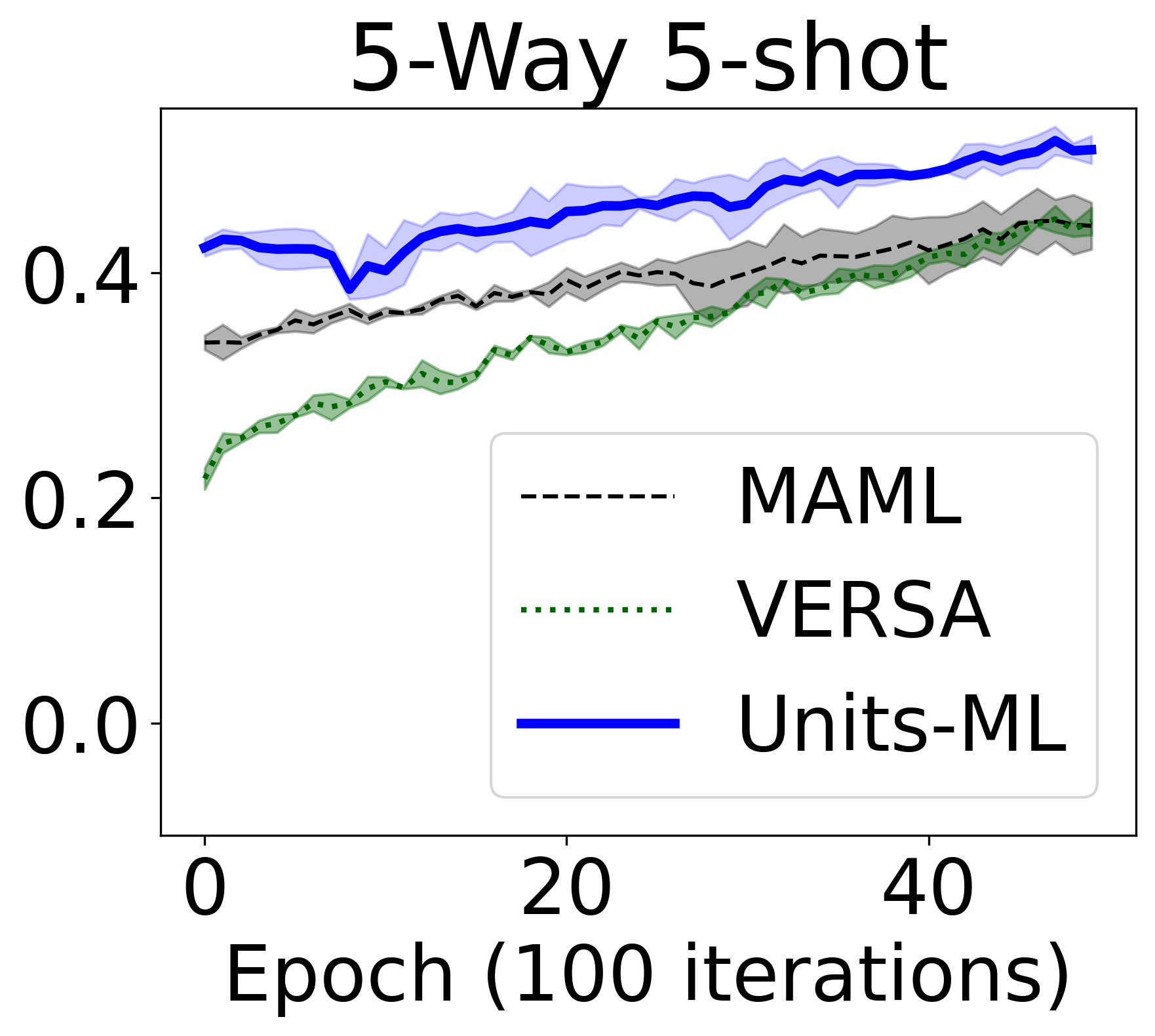}
  \caption{5-w 5-s CifarFS}
  \end{subfigure}
 \begin{subfigure}{0.23\textwidth}
  \centering
  \includegraphics[width=0.9\linewidth]{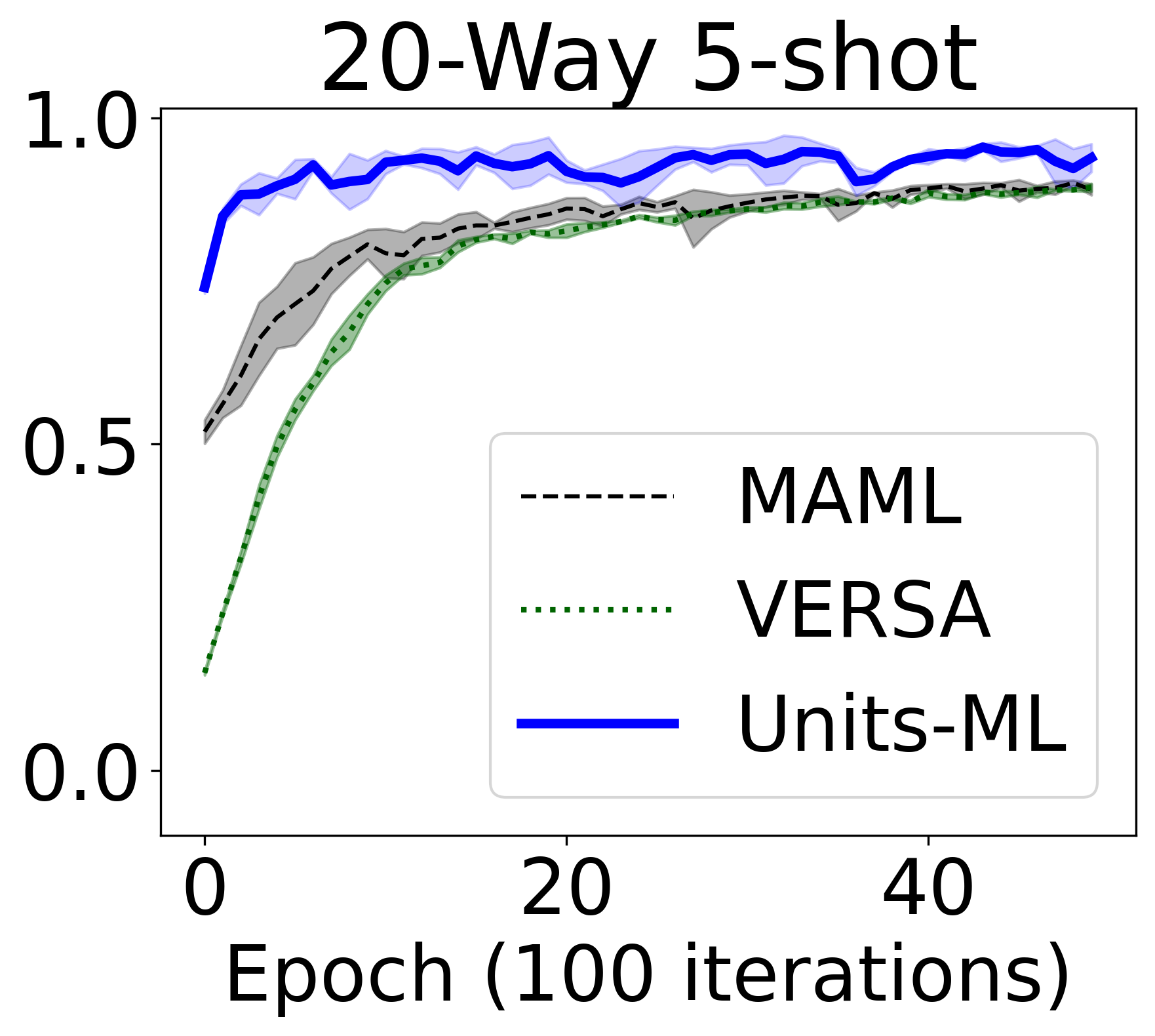}
  \caption{20-w 5-s Omniglot}
\end{subfigure}
\vspace{-2mm}
\caption{
Meta Learning under limited labeling budget}
\label{fig:activetaskSelection}
\vspace{-3mm}
\end{figure}

\begin{table}[h!]
\centering
\small
    \caption{Meta-learning Performance Comparison
    }
    \label{tab:metacomparison}
\begin{tabular}{|p{0.12\textwidth}p{0.14\textwidth}p{0.14\textwidth}|} 
\hline
{\bf Omniglot}&20-way 1-shot(\%)&20-way 5-shot(\%)\\
%&(\%)&(\%)\\
\hline
%Matching Net&93.8&98.5\\
%Proto Net.&95.4&98.7\\
%\hline
MAML&95.8$\pm$0.3&98.9$\pm$0.2\\
Reptile&89.43$\pm$0.14&97.12$\pm$0.32\\
VERSA&97.66$\pm$0.29&98.77$\pm$0.18\\
% Vampire&93.20$\pm$0.28&98.52$\pm$0.13\\
\hline
Units-NTS&91.96$\pm$0.48&97.38$\pm$0.07\\
% Units-ST&93.55$\pm$0.48&98.04$\pm$0.21\\
Units-ML&93.17$\pm$0.25&97.72$\pm$0.18\\
Units-ML $0.2$&96.83$\pm$0.48&99.04$\pm$0.13 \\
Units-ML $0.1$&98.85$\pm$0.83&99.42$\pm$0.25 \\
\hline\hline
 {\bf \textit{mini}-ImageNet}&5-way 1-shot(\%)&5-way 5-shot(\%)\\
 \hline\hline
MAML&48.70$\pm$1.84&63.15$\pm$0.91\\ 
MetaSGD&50.47$\pm$1.87&64.03$\pm$0.94\\
MUMOMAML&49.86$\pm$1.85&-\\
HSML&50.38$\pm$1.85&-\\
CAVIA&51.82$\pm$0.65&65.85$\pm$0.55\\
\hline
LLAMA&49.40$\pm$1.83&-\\ 
PLATIPUS&50.13$\pm$1.86&-\\ 
BMAML&53.17$\pm$0.87&-\\ 
VERSA&53.40$\pm$1.82&67.37$\pm$0.86\\ 
ABML&45.0$\pm$0.6&-\\
% VAMPIRE&51.54$\pm$0.74&64.31$\pm$0.74\\
% \hline
Units-NTS&51.38$\pm$0.33&66.75$\pm$0.42\\
Units-ML&50.86$\pm$0.67&68.16$\pm$0.72\\ 
% \hline
Units-ML $0.2$&61.25$\pm$2.89&80.70$\pm$0.93\\
Units-ML $0.1$&82.26$\pm$4.77&91.07$\pm$0.26 \\
\hline\hline
{\bf CifarFS} & 5-way 1-shot (\%) & 5-way 5-shot  (\%)\\ [0.5ex] 
\hline\hline
MAML&58.9$\pm$1.9&71.5$\pm$1.0\\ 
MetaSGD$^{*}$&57.77$\pm$0.17&71.16$\pm$0.21\\
Reptile&55.86$\pm$1.00&71.08$\pm$0.74\\ % This result is based on the results reported in GCP [15]
VERSA$^{*}$&60.6$\pm$0.68&74.69$\pm$0.29\\
\hline
Units-NTS&59.80$\pm$0.31&76.15$\pm$0.35\\
Units-ML&59.84$\pm$0.11&76.69$\pm$0.44\\
Units-ML $0.2$&76.62$\pm$0.42&83.54$\pm$1.73\\
Units-ML $0.1$&87.92$\pm$0.71&90.46$\pm$1.82\\
\hline
\end{tabular}
\begin{tablenotes}
      \small
      \item * Indicates results from local reproduction
    \end{tablenotes}
\vspace{-3mm}
\end{table}

\subsection{Meta-Learning Performance Comparison}
\vspace{-2mm}
 Meta-learning performance on the three datasets and comparison with state-of-the-art competitive models are presented in Table \ref{tab:metacomparison}. For the comparison, we present the Units-ML (the proposed task selection) and Units-NTS (stands for No Task Selection).  As can be seen, the model with task selection achieves improvement over no task selection in almost all experiments except $5$-way $1$-shot mini-ImageNet experiments where the performance is close. We also present Units-ML with different uncertainty thresholds (\eg, Units-ML $0.2$) to show the flexibility and effectiveness of the predicted uncertainty. For example, Units-ML 0.2 considers the predictive performance using the samples for which the model's predicted uncertainty is less than $0.2$. When considering only the confident predictions by adjusting the uncertainty threshold, our model achieves considerably higher accuracy demonstrating the effectiveness of uncertainty threshold. 

\section{Conclusion}
\vspace{-2mm}
In this paper, we propose an uncertainty-aware optimization-based meta-learning model for few-shot learning. Building upon the theory of subjective logic, the proposed Units-ML model successfully identifies the known uncertainty using vacuity and dissonance and identifies unknown uncertainty using (incorrect) belief mass, respectively. We design a novel task uncertainty score to choose the most informative tasks for meta-training. Our approach achieves comparable performance to many state-of-the-art optimization-based meta-learning methods. We further show the potential of our model for out-of-distribution distribution detection and label efficient task selection. Also, by adjusting the uncertainty threshold, Units-ML can provide a much more reliable prediction performance, which is essential to support decision-making in critical domains. As future work, we plan to extend our framework to metric-based and other meta-learning approaches that train in an episodic fashion.  

\vspace{-1mm}
\section*{Acknowledgements}
\vspace{-1mm}
This research was supported in part by an NSF IIS award IIS-1814450 and an ONR award N00014-18-1-2875. The views and conclusions contained in this paper are those of the authors and should not be interpreted as representing any funding agency. We would also like to thank the anonymous reviewers for their constructive comments.

%%%%%%%%% REFERENCES
{\small
\bibliographystyle{ieee_fullname}
\bibliography{refcv}
}

\afterpage{\blankpage}
\newpage
%\afterpage{\blankpage}
%\newpage

\appendix
%\title{Appendix}
%\maketitle
\onecolumn

\begin{center}
    {\large \bf Appendix}
\end{center}

\paragraph{Organization of the appendix.} In this Appendix, we first prove the relationship between incorrect belief and conflicting belief in {\bf Section \ref{sec:proof_of_theorem}}. After that, we provide additional details of related works and comparison baselines in {\bf Section \ref{sec:related_works_and_baselines}}. We then describe the training process with a complexity analysis in {\bf Section \ref{sec:complexity_analysis_and_tr}}. Finally, we provide additional results on label-efficient meta-learning, illustrative examples on the predicted multidimensional belief, effectiveness of the multidimensional belief showing its potential for detecting OOD and uncertain predictions, comparison with other competitive baselines on any-shot and meta-dataset experiments, and ablation studies in {\bf Section \ref{sec:additional_exp_and_ablation}}. Source code for our experiments is available at
\footnote{Link: \href{https://github.com/pandeydeep9/Units-ML-CVPR-22}{https://github.com/pandeydeep9/Units-ML-CVPR-22}}

\section{Proof Theorem 1}\label{sec:proof_of_theorem}
Before presenting the main proof, we first review some important concepts and show some useful results that will be used in the proof. 
\begin{definition}
Consider we have a sample for which the model outputs $N-$ dimensional belief vector $\mathbf{b} = (b_1,b_2,...b_N)$. Let $S_b = \sum_{i=1}^N b_i$ represent the total belief, $b_{cor}$ represent the correct belief, $ib$ represent the incorrect belief, and $cb$ represent the conflicting belief/dissonance. The conflicting belief $cb$ can be computed from the belief vector as 
\begin{align}
\label{eq:proof_cb}
    &cb = \sum_{k=1}^K \Big( b_k \frac{ \sum_{j\neq k} b_j Bal(b_j, b_k)}{\sum_{j\neq k} b_j } \Big),\\ 
    &Bal(b_j, b_k) = 
\begin{cases}
  1 - \frac{|b_j - b_k|}{b_j+ b_k}, & \text{if } b_i b_j > 0\\
  0, & \text{otherwise}
\end{cases}
\end{align}
where $Bal(\cdot, \cdot)$ is the relative mass balance function between two belief masses. 
\end{definition}

\begin{proposition} \label{proof_prop}
Zero belief masses in the belief vector have no contribution to both the conflicting belief and the incorrect belief. For any two non-zero belief masses $b_i$ and $b_j$, the relative mass balance $Bal(b_i, b_j)$ is given by
\begin{align}
    Bal(b_i,b_j) = \frac{2\times\min(b_i,b_j)}{b_i + b_j}
\end{align}
\end{proposition}
\begin{proposition}\label{lemma1dis}
The conflicting belief ($cb$) is a permutation invariant function over the belief vector ($\mathbf{b}$).
\end{proposition}

\begin{lemma}\label{lemma2ib}
The incorrect belief is an upper bound of $N-1$ belief subsets of the $N$-dimensional belief vector i.e. $ib \geq S_b - b_{max}$ where $b_{max} = max (b_1,b_2,...b_N)$ 
%Ordering of the beliefs in the belief vector does not impact the overall dissonance
\end{lemma}
\begin{proof}
We know, $ib = S_b - b_{cor}$, and $b_{max} = \max(b_1,b_2,...b_N) \geq b_{cor}$. $ \implies ib \geq S_b - b_{max}$
\end{proof}
\paragraph{Proof of Theorem 1.}
\begin{proof}
Consider a sample in the task $t$. Due to Proposition \ref{proof_prop}, we can consider a belief vector with all non-zero beliefs for the proof. Further, as a consequence of Proposition \ref{lemma1dis}, without loss of generality, we can consider that the beliefs in the belief vector are ordered in an descending order i.e. $\mathbf{b} = (b_1, b_2,...b_N), b_{max}=b_1\geq b_2,..\geq b_N$
Now, the conflicting belief can be evaluated as
\begin{align}\label{proof:defn_cb}
    &cb = \sum_{i =1}^{N} b_i \times \bigg(\frac{\sum_{j\neq i }^{N}b_jBal(b_i, b_j)}{\sum_{j\neq i}^{N}b_j}\bigg) \\
    \nonumber&= b_1 \Big( \frac{1}{\sum_{j\neq 1} b_j}\Big) \Big( \frac{2b_2^2}{b_1 + b2} + \frac{2b_3^2}{b_1 + b3} +... \frac{2b_N^2}{b_1 + bN} \Big)\\
    \nonumber&+ b_2 \Big( \frac{1}{\sum_{j\neq 2} b_j}\Big) \Big( \frac{b_12b_2}{b_1 + b2} + \frac{2b_3^2}{b_2 + b3} +... \frac{2b_N^2}{b_2 + bN} \Big) +... \\
    \nonumber& +b_N \Big( \frac{b_N}{\sum_{j\neq N} b_j}\Big) \Big( \frac{b_12b_N}{b_1 + b_N} + \frac{b_22b_N}{b_2+B_N} +... \frac{b_{N-1}2b_N}{b_{N-1} + bN} \Big)
\end{align}
From the above expression, we can see that the numerator does not have a $b_1^2$ term. Considering the terms in dissonance with $2b_2^2$ in numerator, we get
\begin{align*}
    b_2^2 \text{  terms} &= 2b_2^2 \Big( \frac{b_1}{b_1+b_2}\Big) \Big(\frac{1}{\sum_{j\neq 1} b_j} + \frac{1}{\sum_{j \neq 2} b_j}\Big) \\
    &= b2 \times 2b_2 \Big( \frac{b_1}{b_1+b_2}\Big) \Big(\frac{b_1 + b_2 + 2 \times \sum_{j \neq 1,2}b_j}{\sum_{j\neq 1} b_j\sum_{j \neq 2} b_j}\Big) \\
    & \quad \quad \text { as } 2b_2 \leq b_1 + b_2\\
    b_2^2 \text{  terms} &\leq b_2 b_1 \Big( \frac{(2 b_2+ 2 \times \sum_{j \neq 1 ,2 }b_j)}{b_1+b_2 + 2 \times \sum_{j\neq1,2}b_j}\Big) \times \Big(\frac{b_1 + b_2 + 2 \times \sum_{j \neq 1,2}b_j}{\sum_{j\neq 1} b_j\sum_{j \neq 2} b_j}\Big)\\
    &\leq 2 b_2 b_1 \frac{1}{\sum_{j\neq 2} b_j} \leq 2b_2
\end{align*}
Now, considering the terms in dissonance with $2b_n^2, n \in [3, N]$ in numerator, we get
\begin{align*}
    b_n^2 \text{  terms} &= 2b_n^2 \bigg(  \frac{b_1}{b_1+b_n} \big(\frac{1}{\sum_{j\neq 1} b_j} + \frac{1}{\sum_{j \neq n} b_j} \big)+...
    % & + \frac{b_2}{b_2+b_n} \big(\frac{1}{\sum_{j\neq 2} b_j} + \frac{1}{\sum_{j \neq n} b_j} \big) +... \\
     + \frac{b_{n-1}}{b_{n-1}+b_n} \big(\frac{1}{\sum_{j\neq {n-1}} b_j} + \frac{1}{\sum_{j \neq n} b_j} \big) \bigg)\\
    & \quad \quad \text { as } 2b_n \leq b_1 + b_n,..., 2b_n \leq b_{n-1}+b_{n}\\
     b_n^2 \text{  terms} &\leq 2b_n  \frac{b_1}{ \sum_{j\neq n}b_j} + 2b_n \frac{b_2}{\sum_{j\neq n} b_j}+...2b_n \frac{b_{n-1}}{\sum_{j \neq n} b_j} 
    %  b_3^2 \text{  terms} &\leq 2b_3  \frac{b1}{ \sum_{j\neq 3}b_j} + 2b_3 \frac{b_2}{\sum_{j\neq3} b_j} \\
    \leq 2b_n
\end{align*}
%Similarly, we can solve for all $b_k^2, k \in [2,N]$ terms and arrive at corresponding bounds. 
We replace the bounds in the Equation \eqref{proof:defn_cb} and use Lemma \ref{lemma2ib} to get
\begin{align} \label{proof:ibbound}
    cb \leq 2(b_2 + b_3 +... b_N) = 2 (S_b - b_{max}) \leq 2 ib
\end{align}
which proves that for any sample, the incorrect belief is lower bounded by half the conflicting belief. Finally, task incorrect beliefs and the task conflicting beliefs are average of the query instance incorrect beliefs and conflicting beliefs respectably. The relationship in Eqn. \eqref{proof:ibbound} holds true for all query instances proving that the task incorrect belief is bounded by half of the conflicting belief on the task. 
\end{proof}

\section{Details of Related Work and Baselines}\label{sec:related_works_and_baselines}
\paragraph{Subjective Logic Basics}
\begin{figure*}[h!]
\vspace{-3mm}
\centering
\begin{subfigure}{0.23\textwidth}
\centering
\includegraphics[width=0.9\linewidth]{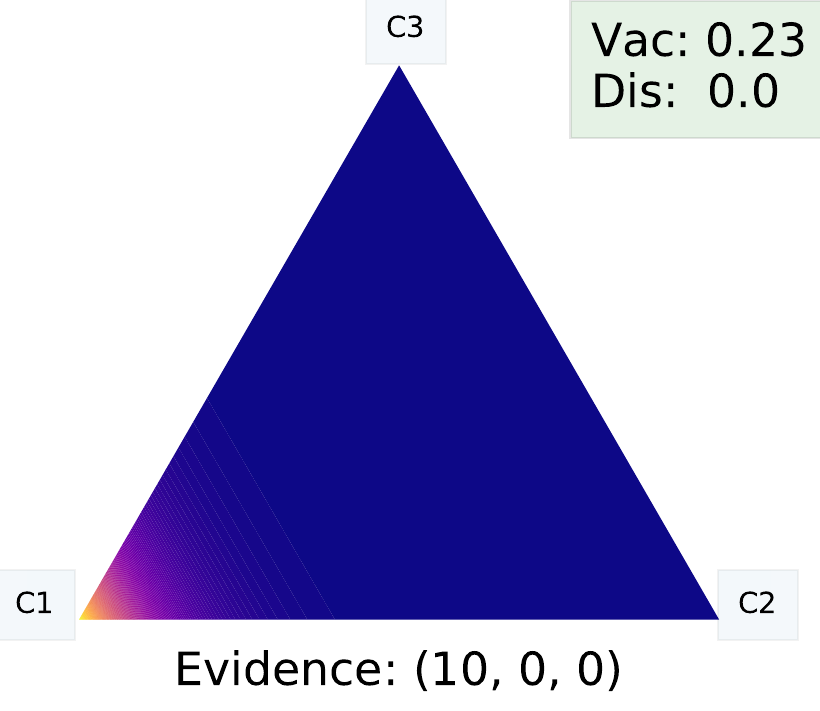}
\caption{A Confident Prediction}
\end{subfigure}
\begin{subfigure}{0.23\textwidth}
%\vspace{0.8mm}
\centering
\includegraphics[width=0.9\linewidth]{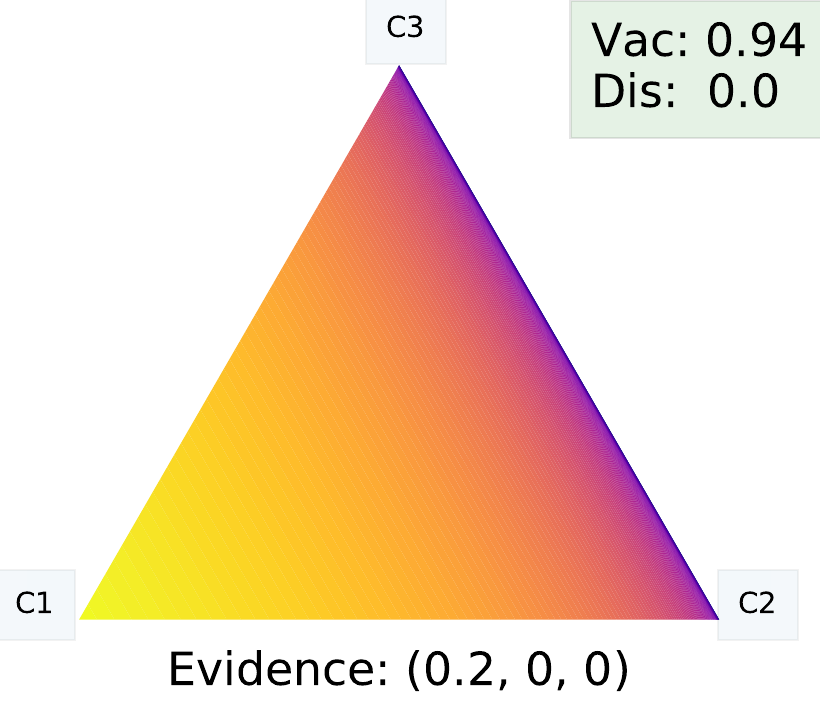}%cifarfs5w1sun.pdf for 5w1s
\caption{High vacuous belief}
\end{subfigure}
\begin{subfigure}{0.23\textwidth}
\centering
\includegraphics[width=0.9\linewidth]{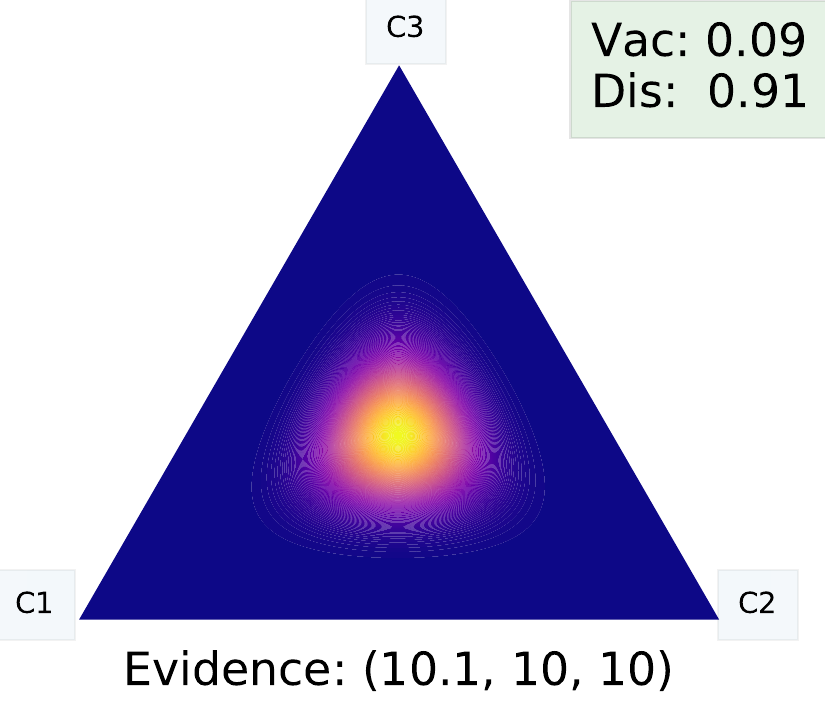}
\caption{High conflicting belief}
\end{subfigure}
\begin{subfigure}{0.23\textwidth}
%\vspace{0.8mm}
\centering
\includegraphics[width=0.9\linewidth]{figures/plotdirichlet/UUfig2-1.png}%cifarfs5w1sun.pdf for 5w1s
\caption{High Total Uncertainty}
\end{subfigure}
\caption{Subjective Logic Opinions: All 4 opinions correspond to a prediction of Class 1 with different uncertainty characteristics. Only the first opinion has low total uncertainty and can be trusted.}
\vspace{-3mm}
\label{fig:subjectiveOpinions_ap}
\end{figure*}

 Subjective Logic (SL) is an extension of probabilistic logic~\cite{josang2016subjective}, which considers the uncertainty in probability assignments along with the probabilities. Recently, using SL, deterministic deep learning (DL) models have been trained to output accurate confidence in predictions along with the predictions for both classification \cite{sensoy2018evidential} \cite{shi2020multifaceted} and regression~\cite{amini2019deep} problems. For classification, the key idea is to train the DL models such that for any input, the model learns to output non-negative evidence for different classes. Using this evidence, the belief for different classes and the model's confidence can be calculated as:
$$
b_k = \frac{e_k}{\sum_{k=1}^K e_k + 1}, \quad v = \frac{K}{\sum_{k=1}^K e_k + 1} , \quad e_k \geq 0
$$
where $e_k$ and $b_k$ represent the evidence and belief for the $k^{th}$ class and $v$ represents the vacuity in the $K-$class classification problem. The vacuous belief (vacuity) is mainly due to the lack of evidence, is greatest when the model outputs no evidence, decreases as the model's evidence increases, and usually indicates unseen/out-of-distribution instances. Complimentary to the vacuous belief, conflicting belief (Eqn \eqref{eq:dis} \ie, the dissonance) \cite{josang2018uncertainty} indicates that the model is confused about the class assignment for the sample and is usually high for noisy/challenging data instances. Based on the conflicting and vacuous beliefs, we can decide which prediction to trust more. Moreover, for predictions with high uncertainty, we can infer the source of the uncertainty and take appropriate actions. Both vacuous belief and conflicting belief are instances of known uncertainty. Finally, there is unknown uncertainty (i.e. the uncertainty that the model is not aware of) that can be estimated from the incorrect beliefs. Unknown uncertainty is mainly due to highly confident incorrect predictions, can be quantified after obtaining the label information, and can only be estimated during training. The evidential models should be trained to minimize as much of the incorrect belief as possible. % confident that the data point belongs to one of the classes, but the model

We present an illustrative description of the subjective logic characteristics in Figure~\ref{fig:subjectiveOpinions_ap} for a 3 class classification (say between apples, mangoes, and oranges). If novel image (say a boat image) is shown to the model, the evidential model can express its lack of knowledge by outputting no evidence or equivalently high vacuity (Figure~\ref{fig:subjectiveOpinions_ap}(b)). Similarly, for an image that is confusing among multiple classes, the model can output high evidence for all the classes leading to high dissonance (Figure~\ref{fig:subjectiveOpinions_ap}(c)). Thus, accurately trained evidential models significantly boost the capabilities of deep learning models as they can identify the source of uncertainty, output level of confidence in predictions, and detect the data noise. Refer to \cite{josang2016subjective} for a thorough study of SL and its characteristics. 
\paragraph{Uncertainty-based Baselines.} Meta-learning has recently been extended to the Bayesian setting \cite{grant2018recasting, gordon2018metalearning, nguyen2020uncertainty}.
Grant et al. \cite{grant2018recasting} developed LLAMA, a bayesian extension of MAML, which used Laplace approximation to learn a distribution instead of the point estimate for task-specific parameters. LLAMA assumes the distribution for task-specific parameters to be gaussian which requires approximating the high-dimensional covariance matrix and thus may not scale to large networks. Finn et al., ~\cite{finn2018probabilistic} presented Probabilistic MAML (PLATIPUS), which addresses the ambiguity of few-shot tasks by using principles of variational inference. Probabilistic MAML relies on a complex training procedure to learn a distribution only for the global parameters and resorts to point estimate for task-specific parameters. Gordan et al., presented VERSA, which uses an amortization network to output a distribution over weights of the base network, learns a distribution over task-specific parameters, and obtains the posterior predictive distribution \cite{gordon2018metalearning}. Kim et al. presented Bayesian MAML, which employs an ensemble of MAML's to obtain uncertainty estimates \cite{yoon2018bayesian}. Uncertainty awareness can be achieved by formulating meta-learning as a problem of inference in a hierarchical bayesian model as in ABML\cite{ravi2018amortized}. ABML uses amortized hierarchical variational inference for the task-specific distribution, learns a point-estimate for the prior distribution (global parameters), and uses Bayes by Backprop to obtain the task-specific distribution. In addition to the above uncertainty-aware meta-learning works, we consider the following baselines:

{MAML \cite{finn2017model}} aims to learn the global parameters such that the model can adapt to new tasks using a few steps of gradient descent. MetaSGD\cite{li2017meta} is an improvement over MAML where both the learning rate and learning direction are learned along with the good initialization. CAVIA \cite{zintgraf2018cavia} is another extension of MAML which addresses the meta-overfitting issue of MAML by separating the model parameters and updating only the subset of the parameters at test time. Reptile \cite{nichol2018first} is a first-order alternative to MAML trained for within task generalization. Finally, {HSML \cite{yao2019hierarchically} and MUMOMAML \cite{vuorio2019multimodal}} are extensions of MAML designed to handle heterogeneous task settings. Moreover, our approach can be applied to most of the optimization-based approaches. We present an extension of our approach to MetaSGD in the additional experiments.

\section{Training Process and Complexity Analysis} \label{sec:complexity_analysis_and_tr}
We aim to learn a good initialization such that for a new task, after learning from the support set, the meta-model can make accurate predictions as well as output the confidence in the prediction. MAML uses softmax activation at the final layer and cross-entropy loss that leads to the Maximum Likelihood Estimation of parameters. In Units-ML, for both local and global updates, we assume that the label for each sample in an N-way K-shot classification problem is obtained from a generative process with a Dirichlet prior: $\text{Dir}({\bm p}_i|\boldsymbol{\alpha}_i)$ and a multinomial likelihood: $\text{Mult}({\bm y}_i|{\bm p}_i)$. 

In particular, for input ${\bm x}_i$ and $\mathbf{y}_i$ as the one-hot representation of the ground truth class, we treat outputs of the neural network as an evidence vector 
${\bm e}_i = (e_i^1,e_i^2,...e_i^N)^\top$. To ensure that the evidence is non-negative, we transform the final layer output by the Softplus function  $\mathbf{e}_i = \ln(1+e^{f(\mathbf{x}_i; \theta)})$. Further, we remove softmax function from our model as it squashes the model outputs in the range $[0,1]$ which is too restrictive for the model evidence. With this, the parameters for the Dirichlet prior are calculated as $\boldsymbol{\alpha}_i = {\bm e}_i +\bm{1}$, following Eqn. \eqref{eq:vac}. 
Similar to \cite{sensoy2018evidential}, we maximize the marginal likelihood obtained from the Dirichlet prior and the multinomial likelihood. 
\begin{align}
\begin{split}
\mathcal{L}_{mar} ({\bm x}_i,{\bm y}_i, \theta_m) = - \ln \Big(\int \text{Mult}({\bm y}_i|{\bm p}_i) \text{Dir}({\bm p}_i|\boldsymbol{\alpha}_i) d {\bm p}_i \Big) \ \
%\text{KL}\big( \text{Dir}({\bm p}_i|\bm{1})||\text{Dir}({\bm p}_i|\boldsymbol{\tilde\alpha}_i)  \big)\\
= -\sum_{j=1}^N y_{i}^j\ln(\frac{e_i^j+1}{\sum_{j=1}^N e_i^j + 1}) 
%+\eta \times \text{KL}\big( \text{Dir}({\bm p}_i|\bm{1})||\text{Dir}({\bm p}_i|\boldsymbol{\tilde\alpha}_i)  \big)
\end{split}
\label{eq:theMarginalLoss}
\end{align}
Additionally, we want to train the model to output no incorrect belief that is achieved by 
using an incorrect belief regularization (Eqn. \eqref{eq:incorrectBelief}) 
\begin{align}
    \mathcal{L}_{ib}({\bm x}_i,{\bm y}_i) =  \mathbf{b}_i \odot (\mathbf{1}-\mathbf{y}_i)
\end{align}
With this, the overall loss for one sample is given by:
\begin{align}
    \label{eq:appendix_overall_loss_sample}
    \mathcal{L}({\bm x}_i,&{\bm y}_i, \theta_m) = \mathcal{L}_{mar} ({\bm x}_i,{\bm y}_i) + \eta \mathcal{L}_{ib}({\bm x}_i,{\bm y}_i)
\end{align}
Here, $\eta$ is the regularization coefficient to balance between minimizing incorrect belief and minimizing the marginal likelihood under the Dirichlet prior, and $\theta_m$ represents the neural network parameters used to output evidence. 

In our meta-learning setup, we consider a batch of tasks at each meta-iteration where each task has a support set and query set. We use the support set loss for local adaptation to task $t$ in the inner loop using $M$ steps of gradient descent as:
\begin{align}
\begin{split}
\theta_0^t &= \theta \quad [\text{Make a copy of global parameters}]\\
\theta_1^t &= \theta_0^t - \alpha \nabla_{\theta_{0}^t} \mathcal{L}_t\big[{f(\theta_{0}^t,X_S^t), Y_S^t}\big] \\
...\\
\end{split}\\
\theta_M^t &= \theta_{M-1}^t - \alpha \nabla_{\theta_{M-1}^t} \mathcal{L}_t\big[{f(\theta_{M-1}^t,X_S^t), Y_S^t}\big]   
\end{align}
At each gradient descent step, we define the support set loss as the average loss of $N\times K$ samples of the support set using the model parameters at that step
\begin{equation}\label{EDLLossSup}
     \mathcal{L}_t\big[{f(\theta_{m-1}^t,X_S^t), Y_S^t}\big] = \frac{1}{N\times K}\sum_{i = 1}^{N\times K} \mathcal{L} ({\bm x}_i,{\bm y}_i, \theta_{m-1}^t)
\end{equation}
We perform this local adaptation for each of the $I$ tasks at each meta-iteration. Next, we use the loss of the adapted model over query set samples to update the global parameters as
\begin{equation}
    \theta \xleftarrow{} \theta - \beta \sum_{t=1}^{I}\nabla_\theta \mathcal{L}_t\big[{f(\theta_{M}^{t},X_Q^t), Y_Q^t } \big]
\end{equation}
The query set loss is defined as the average loss of $N_q$ query set instances with adapted model parameter
\begin{equation}\label{appendix_EDLLoss}
     \mathcal{L}_t\big[{f(\theta_{M}^t,X_Q^t), Y_Q^t}\big] = \frac{1}{N_Q}\sum_{i = 1}^{N_Q} \mathcal{L} ({\bm x}_i,{\bm y}_i, \theta_{M}^t)
\end{equation}

\paragraph{Experiment and Dataset Details.}%for Omniglot and $\eta = \min(0.3,0.03\times E) $
We train the model using a batch of tasks at each iteration where the loss is given by Eqn. \eqref{appendix_EDLLoss}. We consider three datasets: Omniglot, CifarFS and Mini-ImageNet whose details are given in Table \ref{tab:datasetDetails}. We set $\eta=\min(8,0.8 \times E) $ for all Omniglot experiments, $\eta=\min(2,0.2 \times E) $ for CifarFS and 5-way 5-shot mini-ImageNet experiments, and $\eta=\min(0.5,0.05 \times E) $ for 5-way 1-shot mini-ImageNet experiments where $E$ is the epoch number. In the query set, unless specified, we consider 1 instance/class for omniglot and 2 instances/class for all other experiments. We consider a batch of 8 tasks at each iteration for Omniglot experiments, 4 tasks for CifarFS and 5-way 5-shot mini-ImageNet experiments and 2 tasks for 5-way 1-shot mini-ImageNet experiments. We consider 500 iterations/task and start task selection after 5 epochs. In active task selection, we consider a batch of 2 tasks at each iteration and increase $E$ every 100 iterations, start task selection after 100 iterations, and set $\eta=\min(2,0.2 \times E)$. Similar to Antoniou et al. \cite{antoniou2018how} and ALFA \cite{NEURIPS2020_ee89223a}, we learn the batch normalization parameters per step, learn the inner loop learning rate per layer and per step, use an ensemble of top 3 validation set models, and average the results from 3 independently run models to get the final test set performance in Table \ref{tab:metacomparison}. These implementation tricks lead to some improvements for the Units-NTS model, with task selection further improving the generalization results. Furthermore, we start with balancing term $\lambda$ value of $\lambda_{start}=0.99$ and dynamically adjust it  as $\lambda = \lambda_{start} - (\lambda_{start} - \lambda_{end}) \times \min(1.0, E/50)$ as training progresses to reach $\lambda_{end} = 0.5$ at the end of training. For local adaptation, we take 5 gradient descent steps in all Units experiments and 1 gradient descent step in all MetaSGD experiments. 

\begin{table}[htpb]
    \centering
        \caption{Dataset Properties}
    \label{tab:datasetDetails}
\begin{tabular}{||c| c c c||} 
 \hline
 Characteristic&Omniglot&\textit{mini}-ImageNet&CifarFS\\ [0.5ex] 
 \hline\hline
 Image Size&28$\times$28&84$\times$84&32$\times$32\\
 \hline
 Channels&1&3&3\\
 \hline
 Total Classes&1623&100&100\\
  \hline
Tr/Val Split&1150/50&64/16&64/16\\
 \hline
 Images/Class&20&600&600\\
\hline
Augmentation&Rotation&No&No\\
\hline
\end{tabular}
\end{table}
% \begin{algorithm}
%  \caption{Units-ST Task Selection for Efficient Meta-Training}
%  \label{MAMLAlgorithmST}
% \SetAlgoLined
% {\underline{\textbf{Require: }{$\pi ,J$: Initial training and sampled tasks} }}\;
% \textbf{Require: }{$p(\mathcal{T})$: distribution over tasks} \;
% randomly initialize $\theta$ \;
%  \While{ not done}{
%   \eIf{ {  Meta-Iteration $\leq \pi$ }}{
%   Sample $I$ tasks $\mathcal{T}_{i} \sim p(\mathcal{T}) $\;
%   }{
%   {\underline{Sample $J$ tasks $\mathcal{T}_{j} \sim p(\mathcal{T}) $} \;
%       %{\underline{Select I tasks $\mathcal{T}_{i} \sim$ \textbf{TaskSelector( $\mathcal{T}_{j}, \theta, I $) } 
%       %}
%       {\underline{Select $I$ tasks $\mathcal{T}_{i}$ using Eqn. \eqref{eq:score}\;} 
%       }
%       }
%   }
%  \For{\textbf{ all }   $\mathcal{T}_i $}{
% Sample support set $(X_S^i, Y_S^i)$ and query set $(X_Q^i, Y_Q^i)$ from task $\mathcal{T}_i $\;

% Compute task-specific parameters $\theta_M^i$ using M steps of gradient descent over the support set loss\;
%  }
%  Update global parameter $\theta$ using query set loss as  $\theta \xleftarrow{} \theta - \beta \sum_{i=1}^{I}\nabla_\theta L\big[{f(\theta_{M}^{i},X_Q^i), Y_Q^i } \big]$\;
%  }
% \end{algorithm}
%the training process of Units-ST (simple task selection) in Algorithm 1. It involves sampling a large number of tasks ($J$) and selecting $I$ most informative tasks based on Task Uncertainty Score in \eqref{eq:score} at each iteration for meta-training (see Fig. \ref{fig:repIdeaFig}).
\paragraph{Algorithm and Complexity Analysis.}
We present the training process of Units-ML, the proposed task selection method in Algorithm \ref{alg:unitsml}. In Units-ML, we consider $I$ multi-query tasks (see Fig. \ref{fig:repIdeaFig}) with a total of $J$ query sets such that minimal computation is required for determining the task uncertainty score. 
Our complexity analysis of MAML shows that the global update involves calculating Hessian gradient products and is computationally expensive, especially when we take many inner loop updates, or when the network has a large number of parameters. In particular, each additional inner loop update adds a $(I - H_{k})$ term in the outer loop update. The Hessian calculation has an $O(D^2)$ complexity for a model with $D$ parameters (where $D$ is in the scale of thousands--millions in many typical deep neural networks). Even with the use of efficient Hessian-vector multiplication techniques \cite{pearlmutter1994fast}, the computational cost would increase by $O(D)$ for one additional inner loop update. Furthermore, the computation is carried out at each meta-iteration for each task. 
%However, some tasks might not provide any information to the model. 
With task selection, we ensure that tasks are informative for the global parameter update. For task selection, we only need predictions from the adapted model that adds little additional cost as compared to training on a new task. Specifically, task selection is independent of the number of inner-loop gradient descent steps in task adaptation and scales to any number of inner-loop updates without any additional computation. Tasks are selected using the informativeness score of the query set which requires one additional forward pass through the model. This introduces additional computation which scales linearly with the number of query sets considered in task selection. For a baseline model with the computational cost of $O(I\times M \times (F +  B) )$, the cost for the model with task selection is $O(I\times M \times ( F +  B) +J\times F)$ where both models adapt for $M$ steps in the inner loop using $I$ tasks (Task Selection model selects $I$ query sets to be labeled from an unlabeled pool of $J$ query sets), and the computational cost for forward pass through the model $F$ is lower than the computational cost of the backward pass $B$. 

\begin{algorithm}

 \caption{Units-ML Task Selection for Efficient Meta-Training}
 \label{alg:unitsml}
\SetAlgoLined
{\underline{\textbf{Require: }{$\pi, J$: Initial training and sampled tasks} }}\;
\textbf{Require: }{$p(\mathcal{T})$: distribution over tasks} \;
randomly initialize $\theta$ \;
 \While{ not done}{
  \eIf{ {  Meta-Iteration $\leq \pi$ }}{
   Sample $I$ tasks $\mathcal{T}_{i} \sim p(\mathcal{T}) $\;
   }{
   {Sample $I$ multi-query tasks with total of $J$ unlabeled query sets $\mathcal{T}_{i} \sim p(\mathcal{T}) $\;
%   \underline{Sample $J$ tasks such that there are $I$ distinct}
%   \underline{support sets $\mathcal{T}_{j} \sim p(\mathcal{T}) $}\;
%   {Sample support set $(X_S^j, Y_S^j)$ and query set $(X_Q^j, Y_Q^j)$ from task $\mathcal{T}_j $}\;
%   \underline{Group tasks with same support set together}\;
      }
  }
 \For{\textbf{ each }  support set $(X_S^i, Y_S^i)$:}{
Compute task-specific parameters $\theta_M^i$ using M steps of gradient descent over the support set loss\;

Use the adapted model to select query set for each support set
$(X_S^i, Y_S^i)$ using Eqn. \eqref{eq:score}\;
{Label the selected query sets }
 }
 Update global parameter $\theta$ using query set loss 
 {of  $I$ selected query sets as}
 $\theta \xleftarrow{} \theta - \beta \sum_{i=1}^{I}\nabla_\theta L\big[{f(\theta_{M}^{i},X_Q^i), Y_Q^i } \big]$\;
 }
\end{algorithm}
%\newpage\
\section{Additional Experiments and Ablation Study}\label{sec:additional_exp_and_ablation}
In this section, we first present additional results on label-efficient meta-learning. We then present illustrative examples of multidimensional belief quantification that demonstrate the usefulness of our model. Afterwards, we present our model's performance on any-shot datasets and a subset of Meta-Dataset. Finally, we present an ablation study to study the impact of incorrect belief regularization.

\subsection{More Results on Label-Efficient Mata-Learning}
%  \subsection{More Comparison Results for Label-Efficient Mata-Learning}
To demonstrate the effectiveness of using multidimensional belief-based uncertainty measure for task selection, we consider a limited label meta-learning setting. We evaluate the models on a limited labeled budget scenario with a total of 10,000 tasks (each task has 1 instance/class in the query set for omniglot and 2 instances/class for all other datasets). We formulate the task as a multi-query task (with 8 query sets in each task). The baseline MAML model randomly selects the query set to be labeled. VERSA, an uncertainty-aware meta-learning model requests labels for the most informative query set based on the estimated query set uncertainty. Moreover, we extend MetaSGD to be evidential and uncertainty-aware using our proposed approach described in Section \ref{sec:complexity_analysis_and_tr} (referred to as UA-MetaSGD). Both Units and UA-MetaSGD determine the task to be labeled based on the query set informativeness using Equation \eqref{eq:score}. All models are trained with a batch size of 2 for a total of 5000 iterations. Tables \ref{tab:limitedLabelOmni} and \ref{tab:limitedLabelsRest} show the results of the limited labeling budget experiments where the models with task selection show a clear advantage over the models without task selection especially when learning from a limited number of tasks.

\begin{table}[h]
\centering
\small
    \caption{Meta-learning Performance Comparison under Limited Labeling budget Scenario - Omniglot
    %Comparison of our Units-ML model with existing methods. 1 an instance/class is considered in Omniglot query set. Baseline represents the model without task selection, (*) represents results from local reproduction, and Units-ML represents our final model with task selection
    }
    \label{tab:limitedLabelOmni}
\begin{tabular}{|p{0.20\textwidth}p{0.14\textwidth}p{0.14\textwidth}p{0.16\textwidth}|} 
\hline
{\bf Omniglot} 5w 1s&4000 Tasks&8000 Tasks & 10,000 Tasks\\
%&(\%)&(\%)\\
\hline
MAML &73.44&80.18&85.56\\
Versa NTS&74.30&85.70&88.00\\
Versa TS&73.97&85.80&88.23\\
UA-MetaSGD NTS &85.06&88.64&89.76\\
UA-MetaSGD TS &87.08&90.50&91.62\\
Units-NTS &92.38&96.13&96.98\\
Units-ML &95.00&97.90&98.23\\
\hline
{\bf Omniglot} 5w 5s&4000 Tasks&8000 Tasks & 10,000 Tasks\\
%&(\%)&(\%)\\
\hline
MAML &92.00&96.26&96.38\\
Versa NTS&83.97&91.93&94.03\\
Versa TS&84.97&93.30&93.60\\
UA-MetaSGD NTS &93.28&94.96&95.38\\
UA-MetaSGD TS &94.32&97.00&97.48\\
Units-NTS &98.13&98.43&98.46\\
Units-ML &98.86&99.23&99.16\\
\hline
{\bf Omni} 20w 1s&4000 Tasks&8000 Tasks & 10,000 Tasks\\
%&(\%)&(\%)\\
\hline
MAML &65.24&72.41&73.91\\
Versa NTS&70.24&78.33&80.52\\
Versa TS&72.15&79.93&81.85\\
UA-MetaSGD NTS &61.94&71.34&71.44\\
UA-MetaSGD TS &62.38&71.07&72.51\\
Units-NTS&73.35&75.88&83.56\\
Units-ML&79.17&78.60&83.45\\
\hline
{\bf Omni} 20w 5s&4000 Tasks&8000 Tasks & 10,000 Tasks\\
%&(\%)&(\%)\\
\hline
MAML &86.62&89.75&88.26\\
Versa NTS &85.05&89.65&90.67\\
Versa TS &84.55&89.19&90.60\\
UA-MetaSGD NTS &73.65&78.82&79.94\\
UA-MetaSGD TS  &75.09&79.52&81.74\\
Units-NTS &91.32&94.05&95.66\\
Units-ML &93.20&94.16&96.61\\
\hline
\hline
\end{tabular}
\end{table}

\begin{table}[h!]
\centering
\small
    \caption{Meta-learning Performance Comparison under Limited Label Budget - CifarFS and mini-ImageNet
    %Comparison of our Units-ML model with existing methods. 1 an instance/class is considered in Omniglot query set. Baseline represents the model without task selection, (*) represents results from local reproduction, and Units-ML represents our final model with task selection
    }
    \label{tab:limitedLabelsRest}
\begin{tabular}{|p{0.20\textwidth}p{0.14\textwidth}p{0.14\textwidth}p{0.16\textwidth}|} 
\hline

\hline
{\bf CifarFS} 5w 5s&4000 Tasks&8000 Tasks & 10,000 Tasks\\
%&(\%)&(\%)\\
\hline
MAML &29.72&30.14&29.95\\
Versa NTS &32.60&40.11&42.91\\
Versa TS &33.38&42.06&43.75\\
UA MetaSGD NTS &52.18&54.38&58.30\\
UA MetaSGD TS &52.93&55.98&58.30\\
Units-NTS &53.87&58.37&61.30\\
Units-ML &55.88&60.53&61.39\\
\hline
{\bf \textit{mini}-ImageNet} 5w 5s&4000 Tasks&8000 Tasks & 10,000 Tasks\\
%&(\%)&(\%)\\
\hline
MAML &27.23&33.36&35.74\\
Versa NTS &39.45&45.16&45.48\\
Versa TS &37.86&45.48&46.21\\
MetaSGD NTS &44.88&48.95&51.69\\
MetaSGD TS  &45.51&50.35&51.54\\
Units-NTS &43.21&48.32&49.54\\
Units-ML &47.33&48.34&54.23\\
\hline
\hline
\end{tabular}
\end{table}

\subsection{Illustrative Examples of Predicted Multidimensional Belief} \label{empirical_validation_proof}

\paragraph{Mulit-Query Tasks. }We present some qualitative results with a 5-way 2-shot \textit{mini}-ImageNet Multi-Query tasks in Figure~\ref{fig:multi_query_task_illustration} to demonstrate the multidimensional belief characteristics for meta-learning. We assume that we have a limited labeling budget and can label only one of the two query sets. After learning on the support set, for query set 1 (Q1), the model can confidently predict the class labels with low overall task level uncertainties (both vacuous belief $vb$ and conflicting belief $cb$). Q1 may not contribute much to the learning of the global parameters as the query set contains little new knowledge (indicated by low vacuous belief) and the model's class discriminating capabilities seem to be accurate (as indicated by low conflicting belief). If we consider query set 2 (Q2), then the model is highly uncertain about the predictions, with comparatively higher conflict in beliefs and a higher lack of confidence. Labeling Q2 to train the meta-learning model is likely to lead the meta-learning model to better generalization and label-efficient meta-learning.

%We present an illustrative example to show the relationship between the conflicting belief and the incorrect belief \bl{to do}

\begin{figure}[ht!] % VSPACE is not allowed
\centering
\begin{subfigure}{0.46\textwidth}
  \centering
  \includegraphics[width=\linewidth]{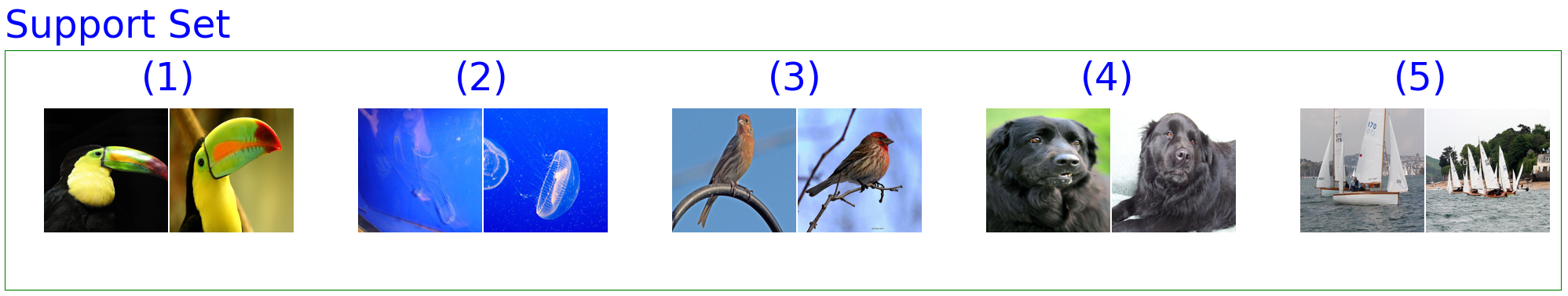}%figures/illustrativeAppendix/multiqsup.png}
%   \caption{5-w 1-s Omniglot}
  \end{subfigure}
\begin{subfigure}{0.46\textwidth}
  \centering
  \includegraphics[width=\linewidth]{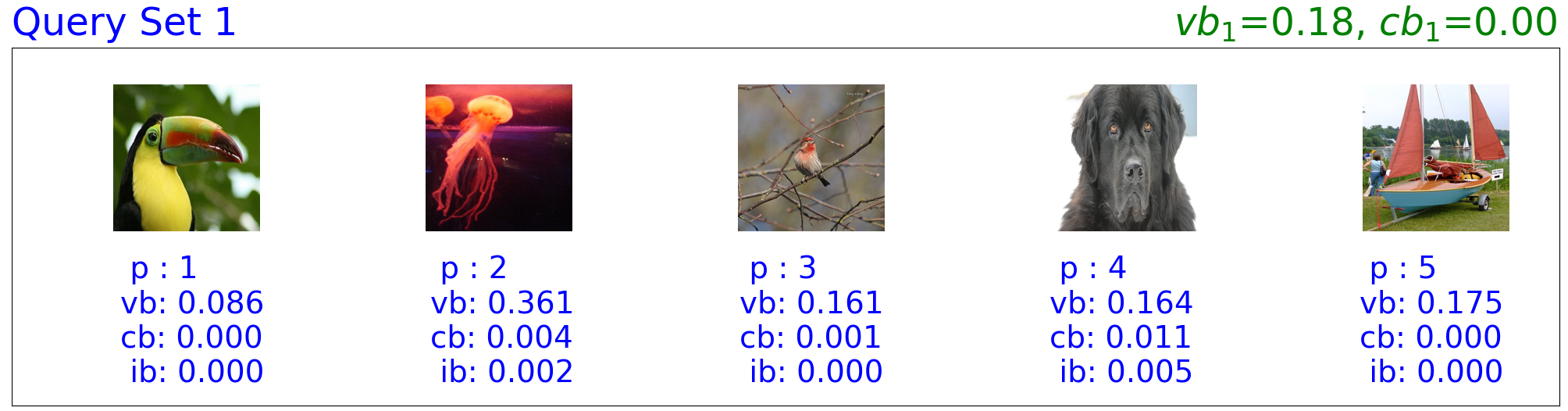}%figures/illustrativeAppendix/multiqq11.png}
%   \caption{20-w 5-s Omniglot}
\end{subfigure}
\begin{subfigure}{0.46\textwidth}
  \centering
  \includegraphics[width=\linewidth]{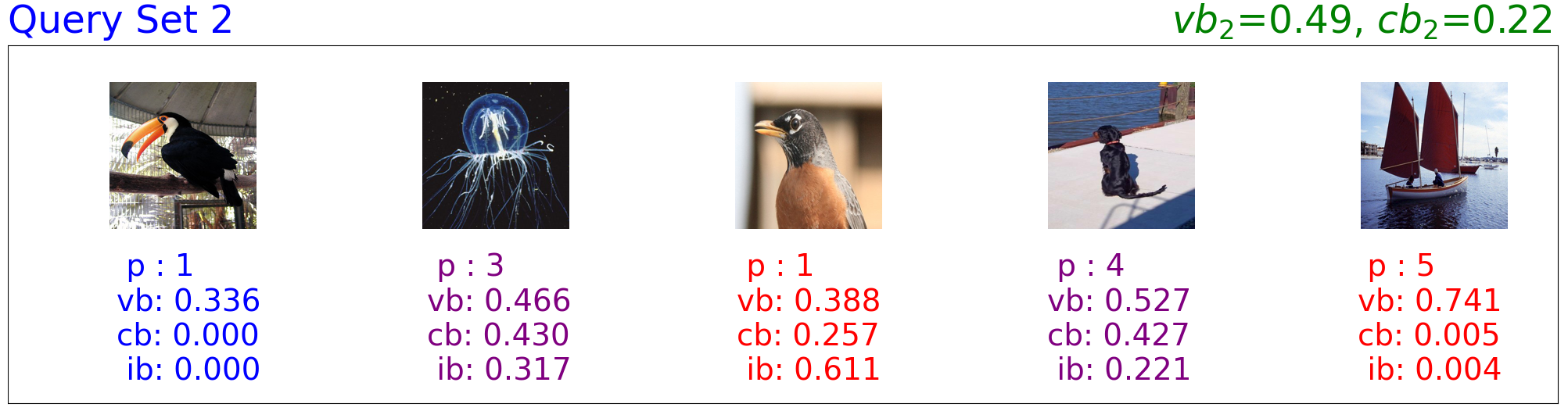}%figures/illustrativeAppendix/multiqq22.png}
%   \caption{20-w 5-s Omniglot}
\end{subfigure}
\vspace{-2mm}
\caption{
Multi-Query Task for Active Task Selection.} %When comparing Query set 1 and Query Set 2, Clearly, the query set 2 has greater uncertainty and the meta-learning model is expected to learn more from the query set 2%carefully created task for illustrative purposes
\label{fig:multi_query_task_illustration}
\vspace{-2mm}
\end{figure} 

\paragraph{Uncertainty Prediction. }We present some additional illustrative examples demonstrating our model's uncertainty prediction capabilities. We trained our model on $5$-way $5$-shot \textit{mini}-ImageNet task and observed the model's behavior on $5$-way tasks. Figure~\ref{illustrative1} shows the model's performance on a $5$-way $1$-shot task. Since each class in the support set has just 1 image/class, there might not be enough evidence in the support set to correctly predict all query set instances. This is reflected by the large model vacuity and dissonance for the query set instances. Further, due to limited evidence in the support set, the predictions for some query instances are wrong. For example, the wolf image in the query set is predicted as a lion. It may be because of the greater match of the orientation of the two animals. As we add instances in the support set that are helpful for the model to correctly classify the query set (see Figures~\ref{illustrative2}, \ref{illustrative3}), the model starts to become be both confident in its prediction as well as correct its prediction indicated by a decrease in the query instance vacuity and dissonance. Further, the vacuity can be useful to detect open-set/OOD instances as shown in Figure~\ref{illustrative4}.

\begin{figure}[h]
        \includegraphics[width=0.56\textwidth]{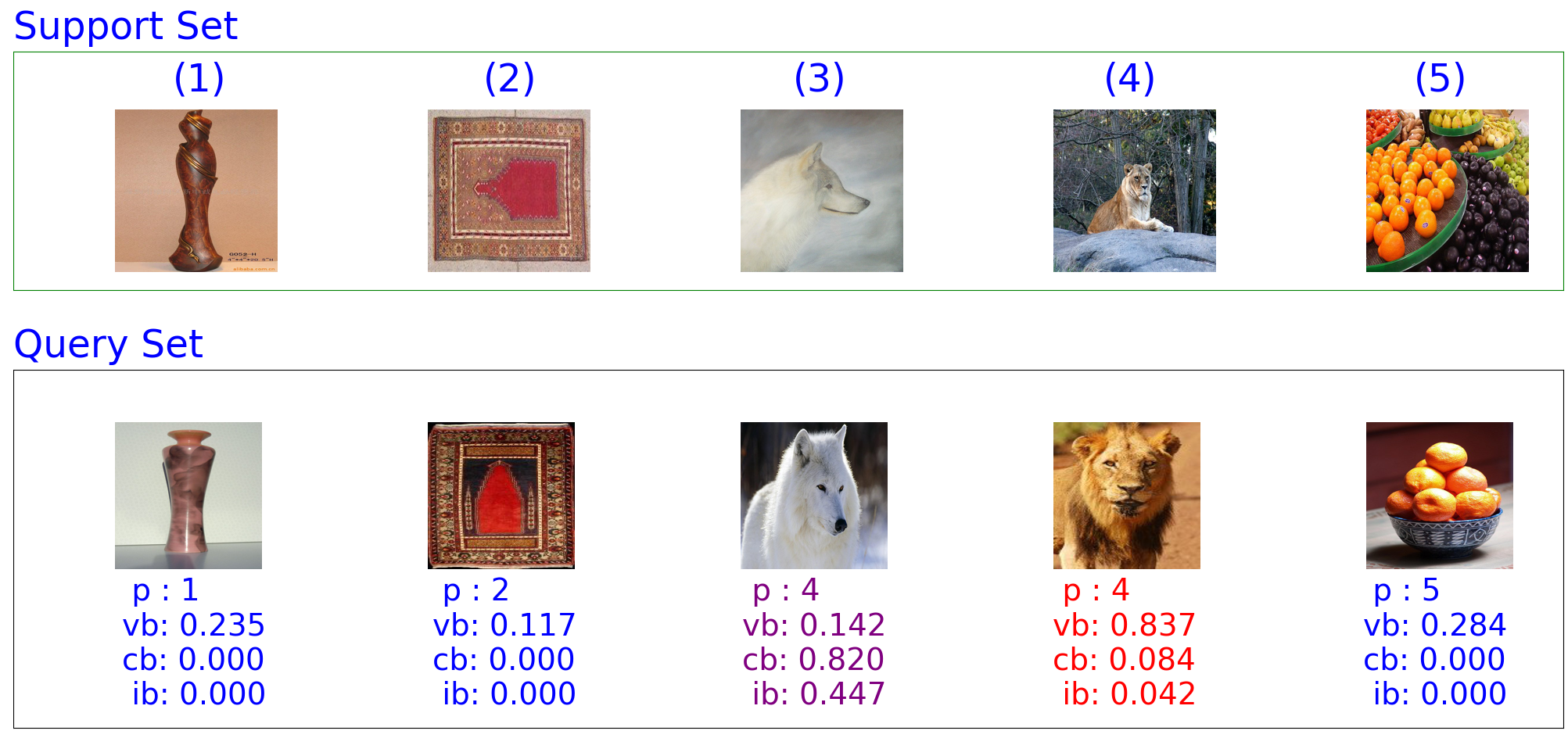}%figures/illustrativeAppendix/5w1smi.png}
        \centering
        \caption{Uncertainty prediction in a 5-w 1-s task.
        }
        \label{illustrative1}
\end{figure}
\begin{figure}[h]
        \includegraphics[width=0.56\textwidth]{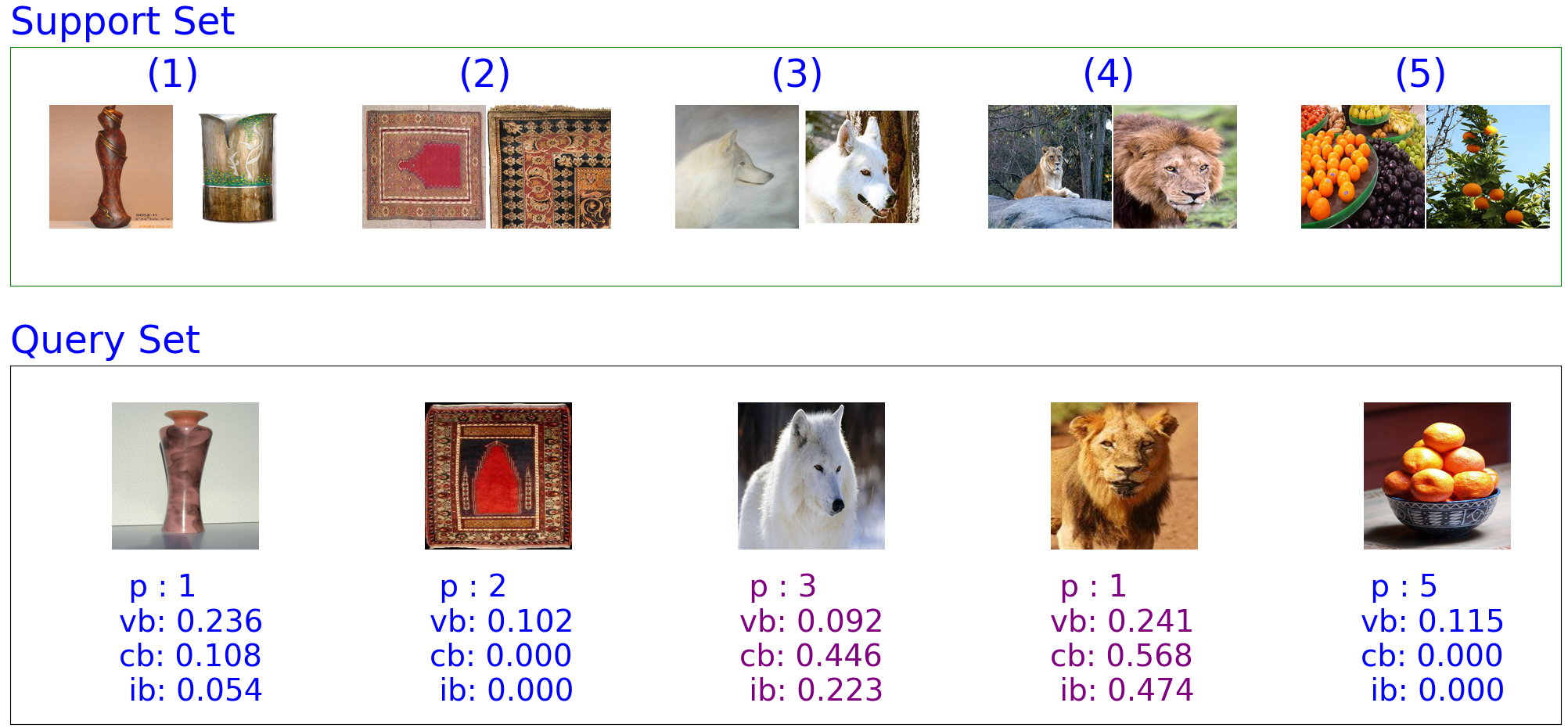}%figures/illustrativeAppendix/5w2smini.png}
        \centering
        \caption{Uncertainty prediction in a 5-w 2-s task.
        }
        \label{illustrative2}
\end{figure}
\begin{figure}[h]
        \includegraphics[width=0.56\textwidth]{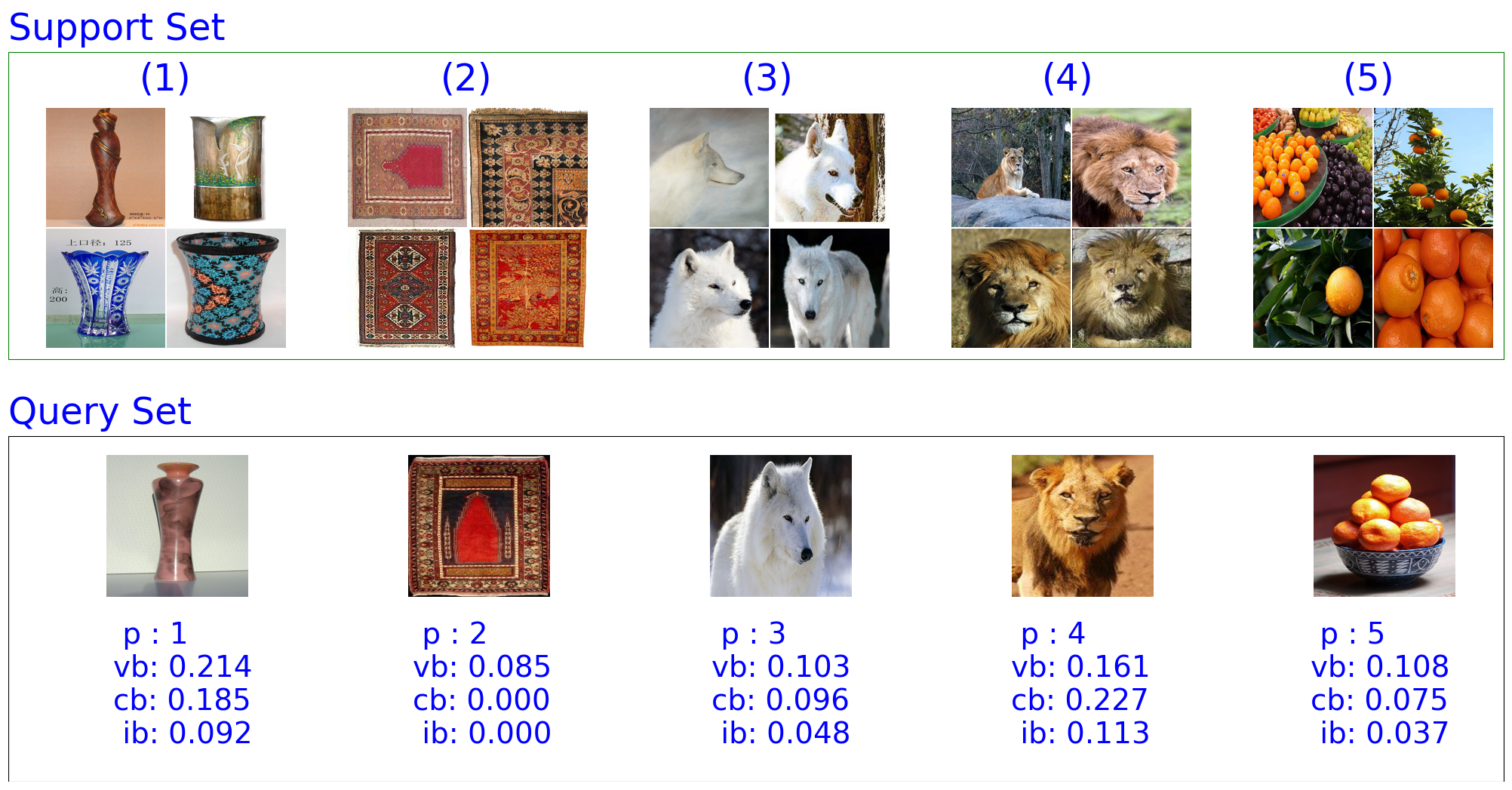}%figures/illustrativeAppendix/5w4smini.png}
        \centering
        \caption{Uncertainty prediction in a 5-w 4-s task.
        }
        \label{illustrative3}
\end{figure}
\begin{figure}[h]
        \includegraphics[width=0.56\textwidth]{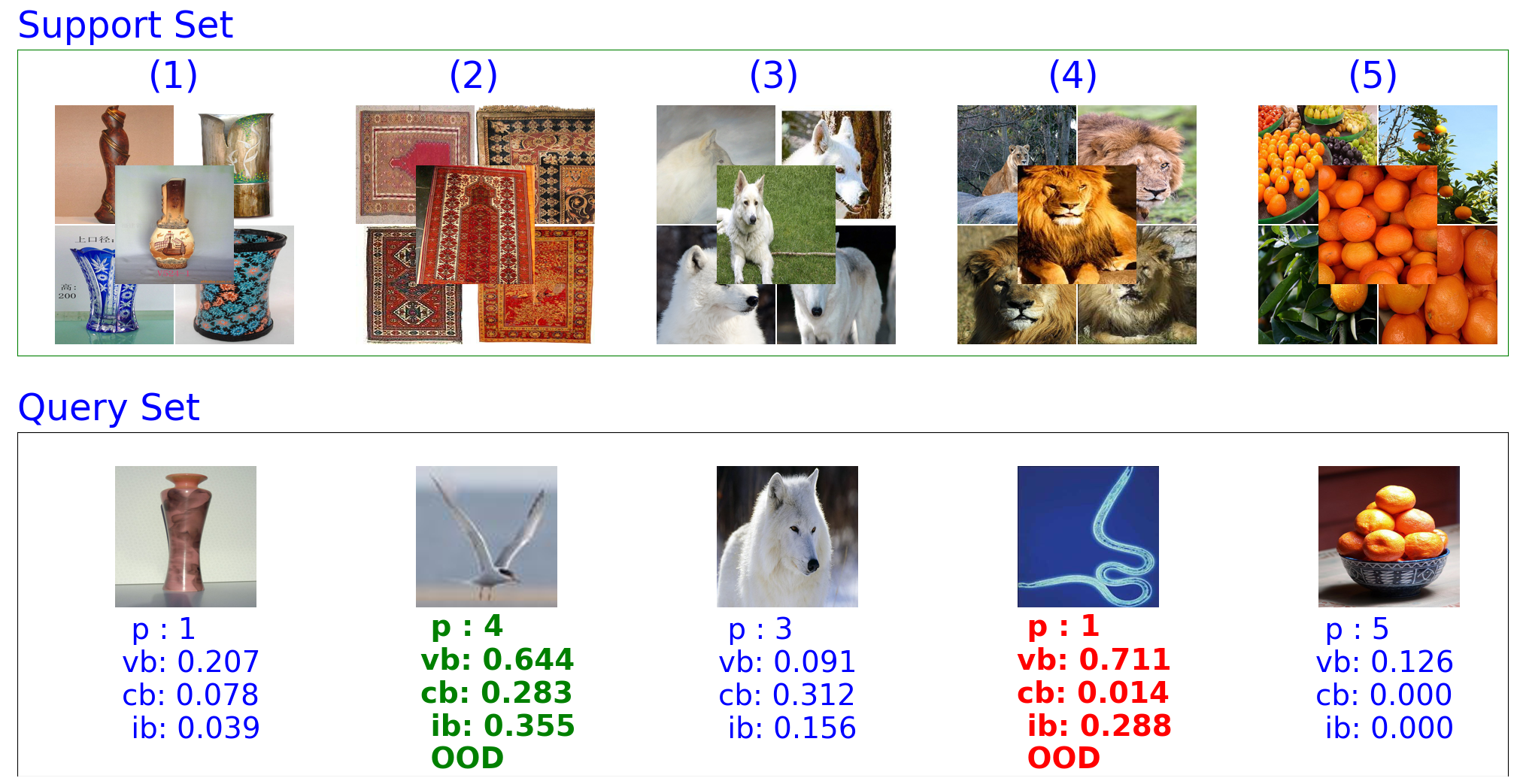}%figures/illustrativeAppendix/5w5soodmini.png}
        \centering
        \caption{Uncertainty prediction in a 5-w 5-s \textit{mini}-ImageNet test task with instances OOD instances in the query set.
        }
        \label{illustrative4}
\end{figure}

\subsection{Effectiveness of Predicted Multidimensional Belief} \label{empirical_validation_proof}
\paragraph{Potential for OOD detection}
We perform experiments over 5-way 1-shot tasks on Omniglot to further demonstrate Units-ML's potential for Out-of-Distribution (OOD) detection.  We train the model on a clean Omniglot dataset for 100 epochs and evaluate the model on 600 test tasks with query set samples rotated by various angles. Table \ref{tab:appendix_omni_5w_results} shows the model's performance on the query set after training for 100 epochs.  Figure~\ref{fig:rotateOnOmni} shows the model's accuracy versus the predicted vacuity for different rotations of query set images. The accuracy drops with larger rotations on query images. Interestingly, vacuity increases proportionally which can be interpreted as: The model is aware of the shift in the distribution of the query set samples. We observe similar behavior with the scaling of the query set instances as shown in Figure~\ref{fig:scaleOnOmni}. Furthermore, in Figure~\ref{fig:rotateOnOmni}, due to the special nature of the Omniglot images (i.e., characters), some of the images (e.g., I,H,O,N,S,X,Z) are less sensitive to a 180$^{\circ}$ rotation. The model accurately recognizes this and reports a low vacuity around that angle. 
\begin{figure}[h]
\centering
\begin{minipage}{0.49\textwidth}
\label{fig:omni_scale_rot_ablation}
\centering
\begin{subfigure}{0.49\textwidth}
  \centering
  \includegraphics[width=\linewidth]{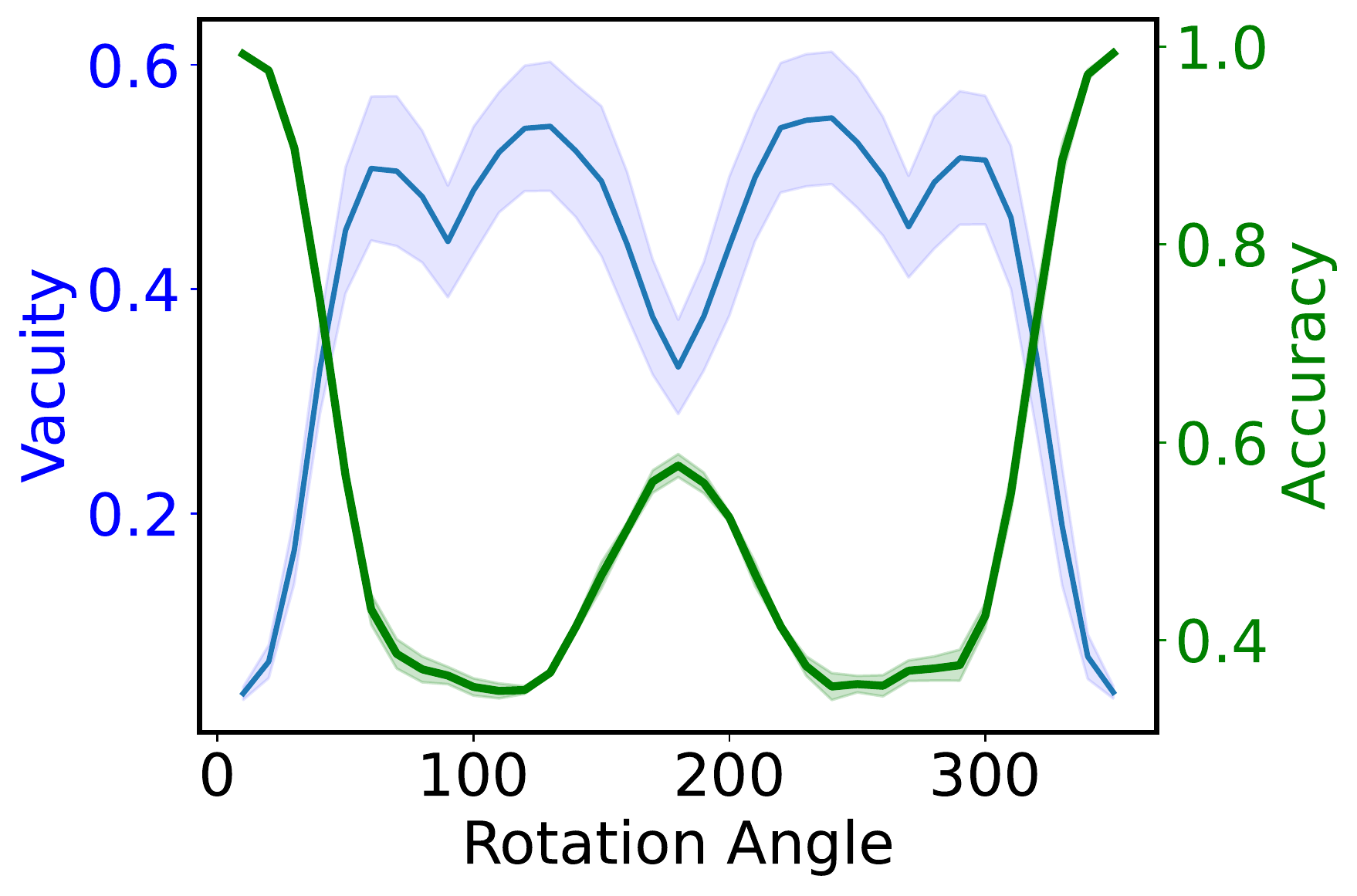}%old/results/appendix/RotationOnOmniglot.pdf}
  \caption{Rotated query set}
  \label{fig:rotateOnOmni}
\end{subfigure}
\begin{subfigure}{0.49\textwidth}
  \centering
 \includegraphics[width=\linewidth]{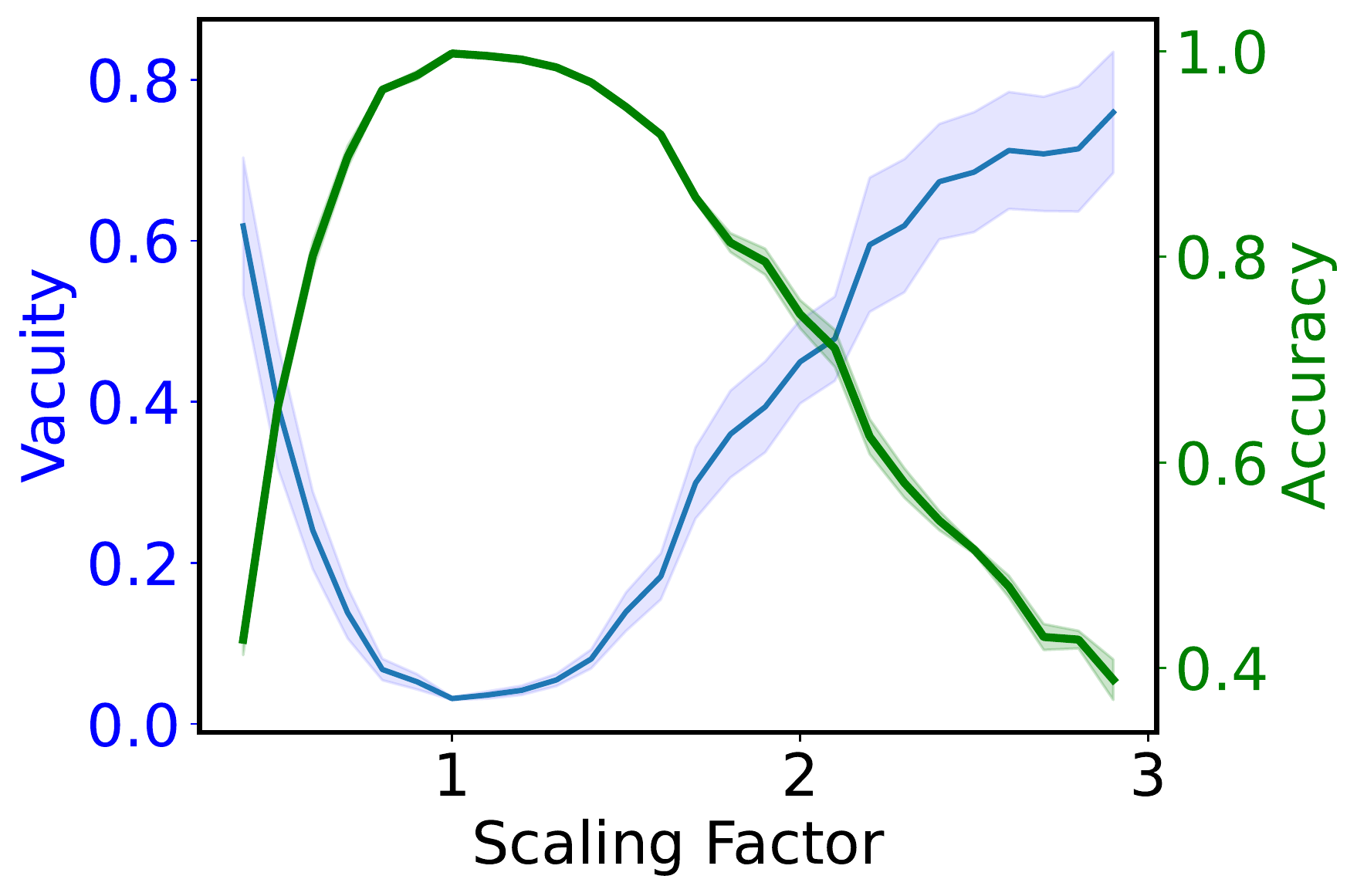}%old/results/appendix/ScalingOnOmniglot.pdf}
  \caption{Scaled query set}
  \label{fig:scaleOnOmni}
\end{subfigure}
\vspace{-2mm}
\caption{Vacuity and accuracy trends in OOD detection}
\end{minipage}
\begin{minipage}{0.46\textwidth}
\captionsetup{type=table} %% tell latex to change to table
\begin{tabular}{|c|c|c|} 
\hline
{\bf Omniglot}&5-way 1-shot(\%)&5-way 5-shot(\%)\\
\hline \hline
MAML&98.7$\pm$0.4&99.1$\pm$0.1\\
Reptile&97.68$\pm$0.04&99.48$\pm$0.06\\
VERSA&99.70$\pm$0.20&99.75$\pm$0.13\\
% Vampire&98.43$\pm$0.19&99.56$\pm$0.08\\
\hline
Units-NTS&99.20$\pm$0.21&99.66$\pm$0.08\\ %used eta = 8.0 
% Units-ST&99.74$\pm$0.02&99.93$\pm$0.07\\ 
Units-ML&99.59$\pm$0.06&99.83$\pm$0.01\\[0.1ex] %Previous Mistake
\hline
\end{tabular}
\caption{Meta Learning Performance Comparison}\label{tab:appendix_omni_5w_results}
\end{minipage}
\end{figure}

\paragraph{Dataset wide statistics for OOD.} %As seen in Appendix Figure 13 (lines 1879-1887), the model correctly outputs large vacuity for OOD instances across the test tasks. Here, 
We conduct additional experiments with 5-way 5-shot CifarFS-Aircraft and \textit{mini}-ImageNet-CUB settings, where we consider the averaged performance across 600 test tasks. After training models on clean tasks from CifarFS and \textit{mini}-ImageNet datasets, we compare the model's performance on the OOD query set constructed from Aircraft/CUB datasets with the In-distribution (InD) query set instances from CifarFS/\textit{mini}-ImageNet, respectively. Specifically, during the test phase, we consider the support set from CifarFS/\textit{mini}-ImageNet datasets and evaluate on the InD and OOD query sets. On average, the vacuity of InD query set is considerably lower than the vacuity of the OOD query set, which further justifies the potential of our model for OOD detection.
\begin{table}[htpb]
    \centering
    %\label{tab:OODDatasetWide}
    \vspace{-2mm}
{
\begin{tabular}{|c| c c c|} 
 \hline
 Dataset&Accuracy&InD Vacuity&OOD Vacuity\\
 \hline
 
 CifarFS&76.69$\%$&0.18&0.45\\ % 5 way 5 shot
 \textit{mini}-ImageNet&68.16$\%$&0.29&0.48\\ % 5 way 5 shot
\hline
\end{tabular}}
\vspace{-2mm}
\end{table}

\paragraph{Comparison with VERSA} %Versa, Units-ML, Units-ST
We also compare our Units-ML model with VERSA, another uncertainty aware model at different uncertainty thresholds for CifarFS dataset. We consider 5-way 5-shot CifarFS tasks where we use the output uncertainty from the two models to obtain top $T\%$ confident predictions over the query set of 600 test tasks and compare the performance. We use the variance of the predictions in VERSA to estimate the uncertainty. Units-ML achieves a higher prediction accuracy in all cases, which suggests that Unit-ML's predicted evidence-based uncertainty is more trustworthy. The VERSA model requires computationally expensive sampling to quantify uncertainty for each query set prediction. Moreover, using ideas from our work, the VERSA model can also be extended to be an computationally efficient evidential meta-learning model. We leave this extension as a future work.  
\begin{table}[htpb]
    \centering
    %\caption{Versa vs Units-ML on 5-way 5-shot CifarFS}
    \label{tab:versaVsUnitsRebuttal}
    \vspace{-3mm}{
\begin{tabular}{|c|c c c c |} 
 \hline
 Model&T=100$\%$&T=$70\%$&T=$60\%$&T=$50\%$\\
 \hline
VERSA&74.69&78.33&81.41&84.33\\ % 5 way 5 shot
Units-ML&76.50&83.26&86.36&87.23\\ % 5 way 5 shot
\hline
\end{tabular}}
\vspace{-3mm}
\end{table}
\newpage
\subsection{Any-Shot and Multi-Dataset Experiments}
We also evaluated our Units model with any-shot classification tasks and multi-dataset settings using task/experiment setup as described in Lee et al. \cite{lee2020learning}. In any-shot experiments, we trained and evaluated on 5-way any-shot tasks with 15 instances/class in the query set having both class and task imbalance. A \textit{mini}-ImageNet trained model was meta-tested on \textit{mini}-ImageNet and CUB whereas a CIFARFS trained model was meta-tested on CIFARFS and SVHN test tasks. In multi-dataset experiments, the model was meta-trained using uniformly sampled 10-way any-shot tasks from Aircraft, QuickDraw, and VGG-Flower datasets and evaluated on Fashion-MNIST and Traffic Signs along with Aircraft, QuickDraw, and VGG-Flower datasets (tasks from Fashion-MNIST and Traffic Signs are not available to the model during meta-training phase). The results of the any-shot and multi-dataset experiments are presented in Table \ref{tab:anyShotRes} and Table \ref{tab:multidatasetres}. For any-shot experiments, our model easily outperforms all the baselines except for Bayesian TAML \cite{lee2020learning}. In multi-dataset settings, the results are slightly different where our model outperforms all the baselines in three of the datasets: QuickDraw, Fashion-MNIST, and VGG-Flower. In the remaining two datasets, our model has comparable performance to other baselines and a slightly lower performance compared to Bayesian TAML. It is worth noting that Bayesian TAML is specifically designed to handle any-shot tasks with class and task imbalance but this is not the design goal of our model. Furthermore, as shown by Units-NTS 0.2/0.1 and Figure~\ref{fig:multi_vac_threshold_trend}, if we consider the uncertainty threshold, our model can outperform all the baselines in all the settings. 
%Such confidence based uncertainty threshold can play a crucial role in real world applications where the tasks in which model's predictions with greatest confidence (i.e. lowest vacuity/uncertainty) can be trusted whereas for tasks with high uncertainty (low confidence), a

\begin{table*}[h]
\centering
\small
    \caption{Any-Shot Setting Comparison
    }
    \label{tab:anyShotRes}
\begin{tabular}{|p{0.16\textwidth}p{0.22\textwidth}p{0.02\textwidth}|p{0.22\textwidth}p{0.02\textwidth}|} 
\hline
{\bf Meta-Training}&\textit{mini}-ImageNet&&CifarFS&\\
\hline 
\end{tabular}
\begin{tabular}{|p{0.16\textwidth}p{0.12\textwidth}p{0.12\textwidth}|p{0.12\textwidth}p{0.12\textwidth}|} 
\hline
{\bf Meta-Testing}&\textit{mini}-ImageNet&CUB&CifarFS&SVHN\\
\hline \hline
%\hline
MAML&66.64&65.77&71.55&45.17\\
Meta-SGD&69.95&65.94&72.71&46.45\\
fo-Proto-MAML&68.96&61.77&71.80&40.16\\
Bayesian TAML&71.46&71.71&75.15&51.87\\
\hline
Units-NTS&71.70&67.95&76.76&52.11\\ 
Units-NTS 0.2&83.27&80.01&84.60&70.63\\ 
Units-NTS 0.1&92.18&88.82&92.00&81.47\\ 
\hline
\end{tabular}
\end{table*} 
\begin{table*}[h!]
\centering
\small
    \caption{Multi-Dataset Setting Comparison
    }
    \label{tab:multidatasetres}
\begin{tabular}{|p{0.16\textwidth}p{0.39\textwidth}p{0.03\textwidth}p{0.03\textwidth}p{0.12\textwidth}p{0.03\textwidth}|}%p{0.05\textwidth}|
\hline {\bf Meta-Training}&Aircraft, QuickDraw, and VGG-Flower & & & &\\%&
\end{tabular}
\begin{tabular}{|p{0.16\textwidth}p{0.12\textwidth}p{0.12\textwidth}p{0.12\textwidth}|p{0.12\textwidth}|p{0.12\textwidth}|}%p{0.12\textwidth}
\hline
{\bf Meta-Testing}&Aircraft&QuickDraw&VGG-FLower&Traffic Signs&Fashion-MNIST\\%&Traffic-Signs
\hline \hline
%\hline
MAML&48.60&69.02&60.38&51.96&63.10 \\
Meta-SGD&49.71&70.26&59.41&52.07&62.71 \\
fo-Proto-MAML&51.15&69.84&65.24&53.93&63.72\\
Bayesian TAML&54.43&72.03&67.72&64.81&68.94\\
\hline
Units-NTS&46.88&73.82&70.52&53.90&69.09\\
Units-NTS 0.2&54.64&82.00&76.84&61.74&74.49\\ 
Units-NTS 0.1&63.43&90.06&82.78&70.60&81.49\\
\hline
\end{tabular}
\end{table*} 

 \begin{figure}[h!]
\begin{subfigure}{0.19\textwidth}
  \centering
  \includegraphics[width=0.9\linewidth]{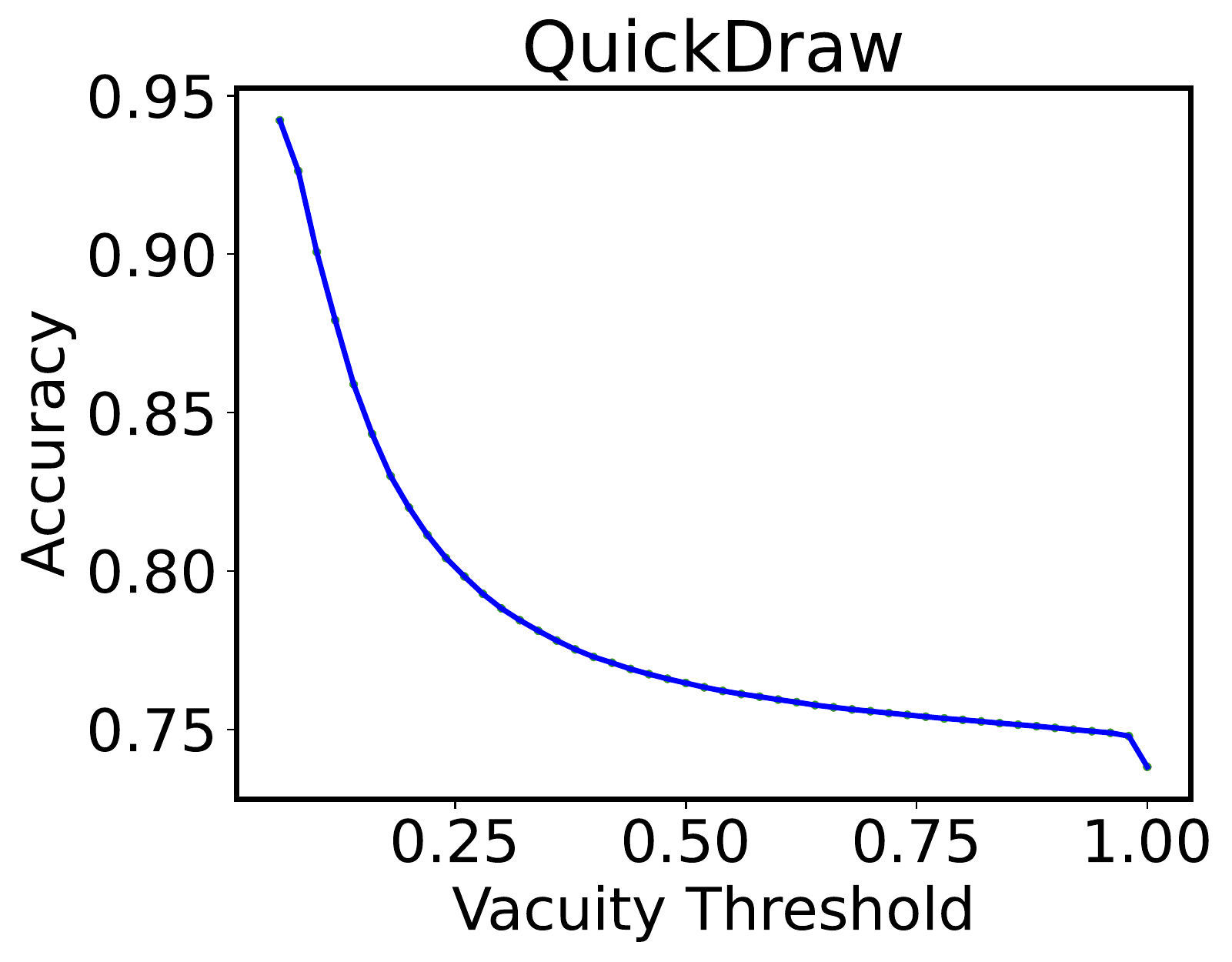}
  \caption{QuickDraw}
\end{subfigure}
\begin{subfigure}{0.19\textwidth}
  \centering
  \includegraphics[width=0.9\linewidth]{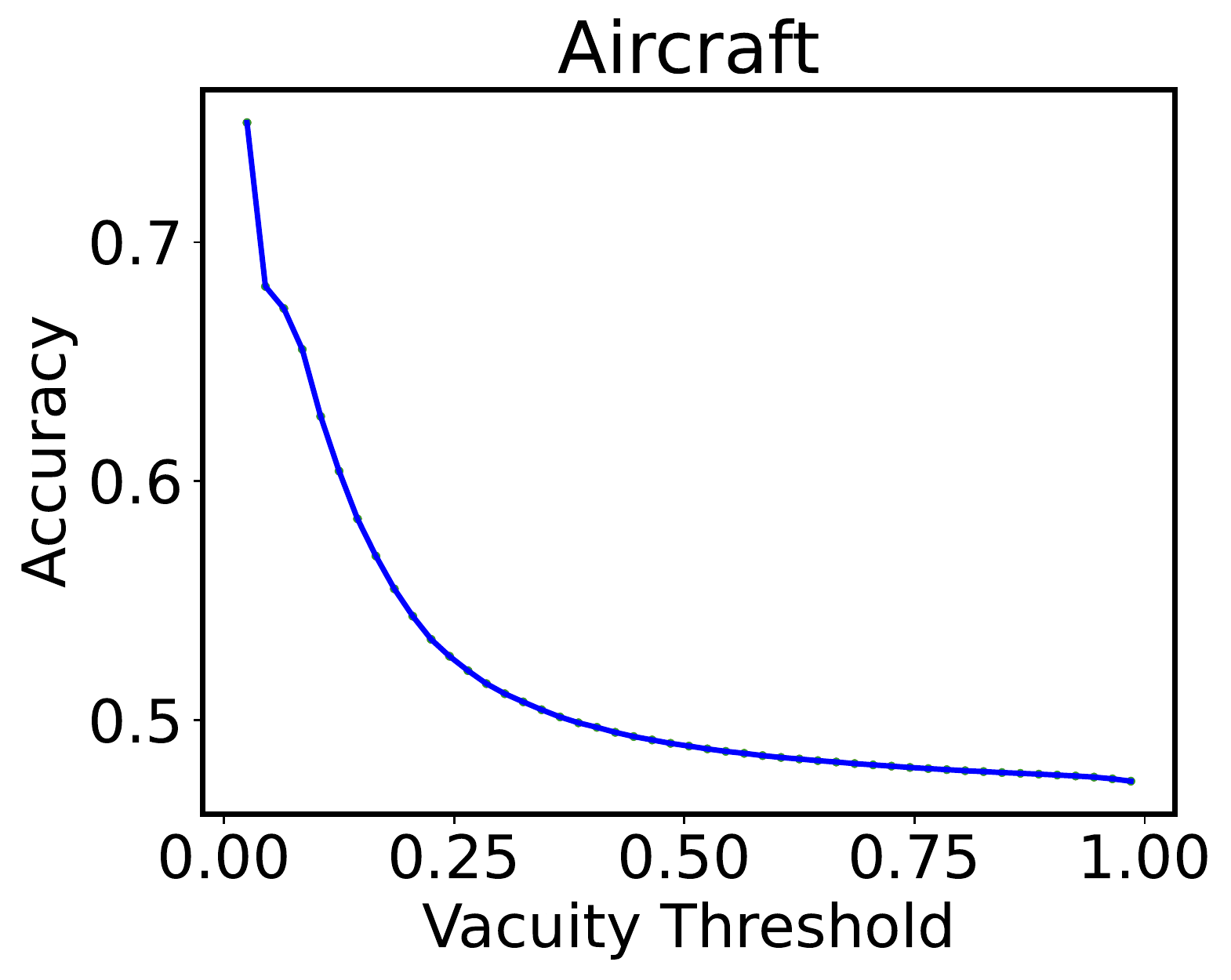}
  \caption{Aircraft}
\end{subfigure}
\begin{subfigure}{0.19\textwidth}
  \centering
  \includegraphics[width=0.9\linewidth]{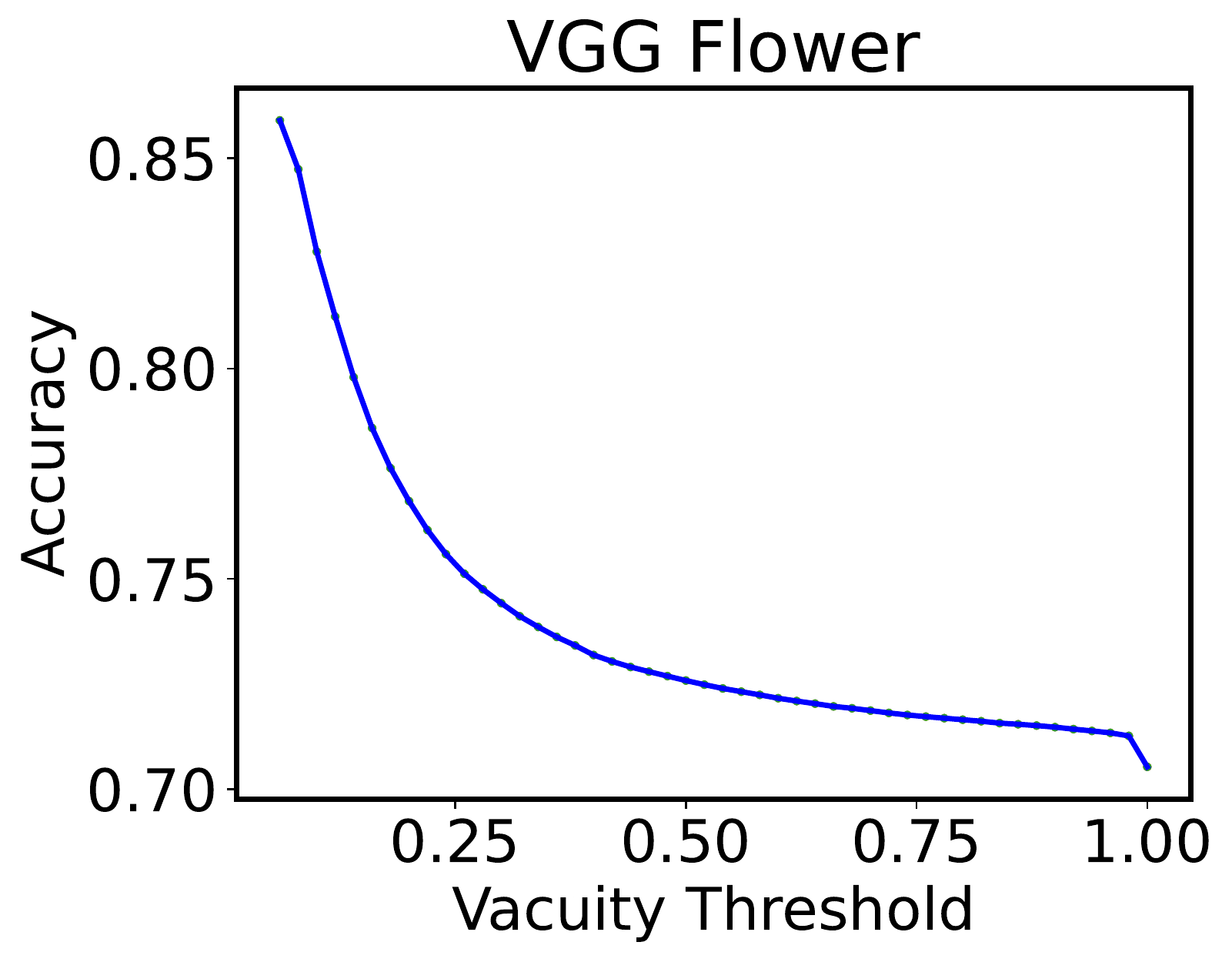}
  \caption{VGG Flower}
\end{subfigure}
\begin{subfigure}{0.19\textwidth}
  \centering
  \includegraphics[width=0.9\linewidth]{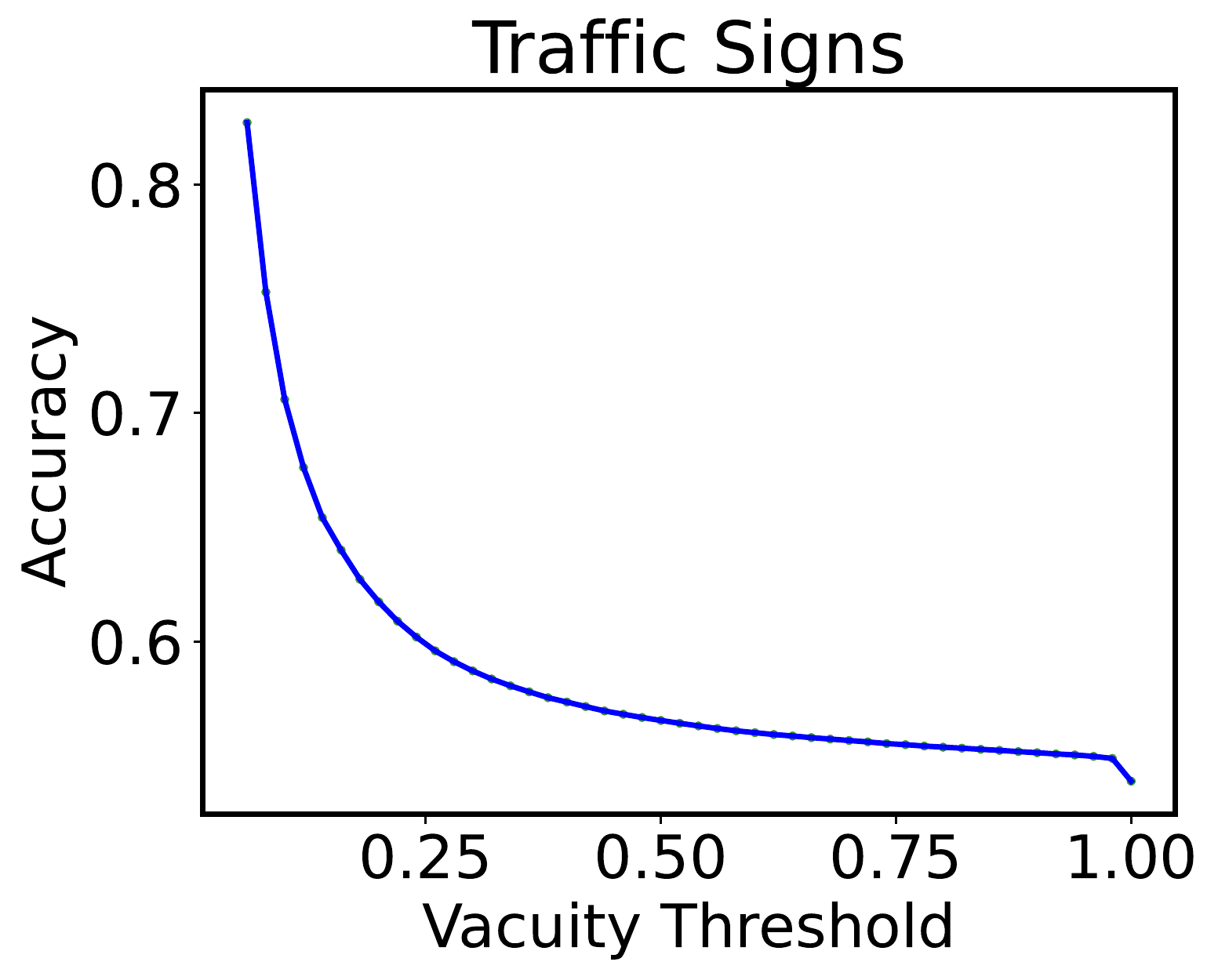}
  \caption{Traffic Signs}
\end{subfigure}
\begin{subfigure}{0.19\textwidth}
  \centering
  \includegraphics[width=0.9\linewidth]{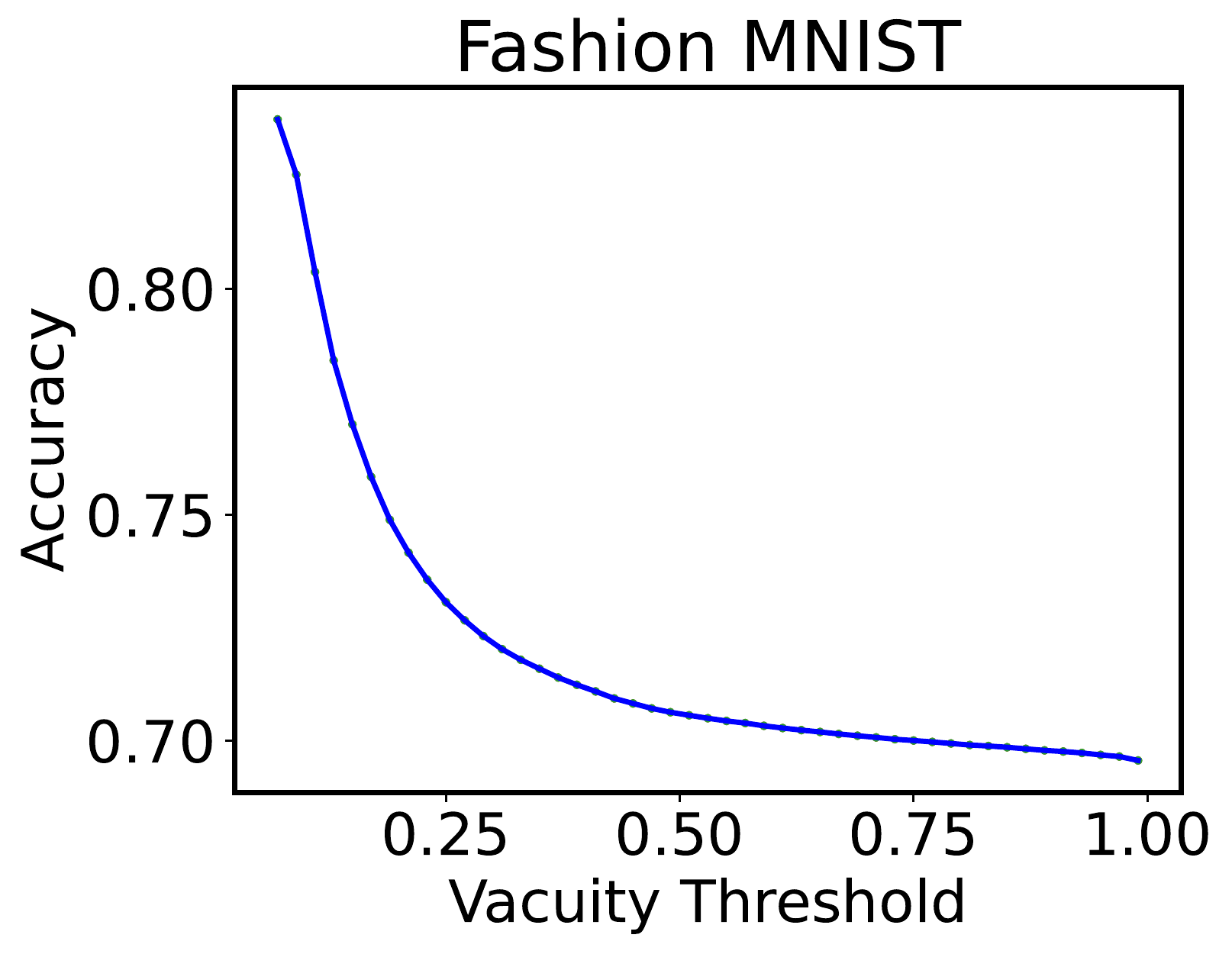}
  \caption{Fashion MNIST}
\end{subfigure}
\caption{
Impact of Vacuity Threshold for Different Datasets in Meta-Dataset}
\label{fig:multi_vac_threshold_trend}
\end{figure}

\subsection{Effect of Incorrect Belief Regularization}
 We add a belief regularization term (Eqn. \eqref{eq:appendix_overall_loss_sample}) to encourage our model to output low (ideally no) belief for classes other than the ground truth and ensure low incorrect belief. The effect of the regularization term is controlled by $\eta = \min(p, p * E/10)$ (Eqn. \eqref{eq:appendix_overall_loss_sample}) where we $p$ is a hyperparameter. Here, we study the impact of this hyperparameter on our model's accuracy, vacuous belief, and incorrect belief. Figure~\ref{fig:appendix_trend_eta_all} shows the impact of regularization on training accuracy, validation accuracy, vacuity, and incorrect belief for a 5-way 1-shot CifarFS experiment. If there is low/no regularization, then the model outputs high confidence even for wrong predictions as indicated by large incorrect beliefs. When the incorrect belief regularization dominates the loss, the model outputs high vacuity for all tasks and the model fails to train properly. The model shows the best performance when there is a good balance in penalizing incorrect belief (through belief regularization) and encouraging large correct belief (through the loss term).%\bl{to do}
 \begin{figure}[h!]
\centering
\begin{subfigure}{0.24\textwidth}
  \centering
  \includegraphics[width=\linewidth]{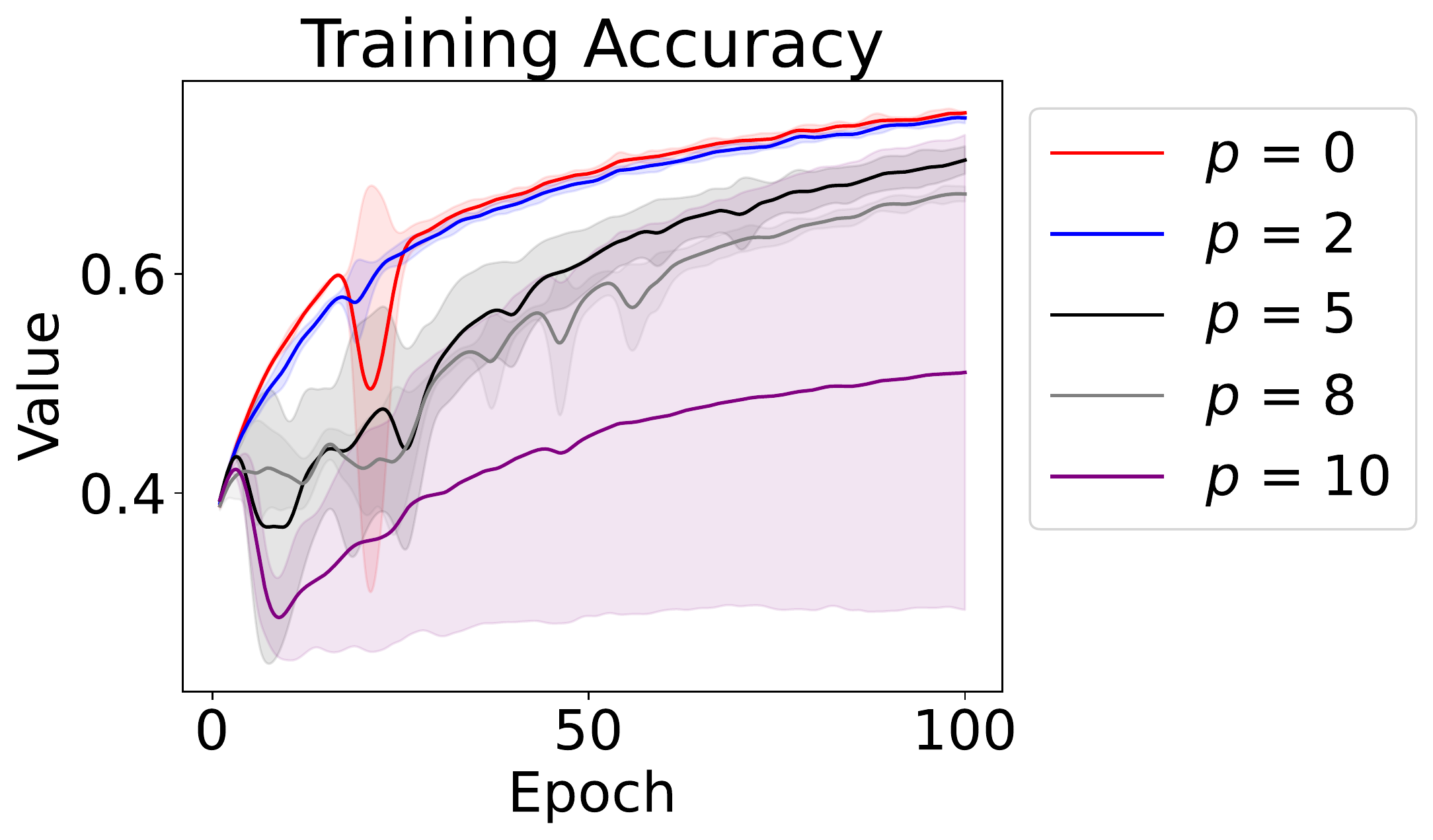}%old/results/newnov21/mininov16Tr5W5S2.0Accuracy.pdf}
  \caption{Training}
  \label{fig:appendix_tr_eta_trend}
\end{subfigure}
\begin{subfigure}{0.24\textwidth}
  \centering
  \includegraphics[width=\linewidth]{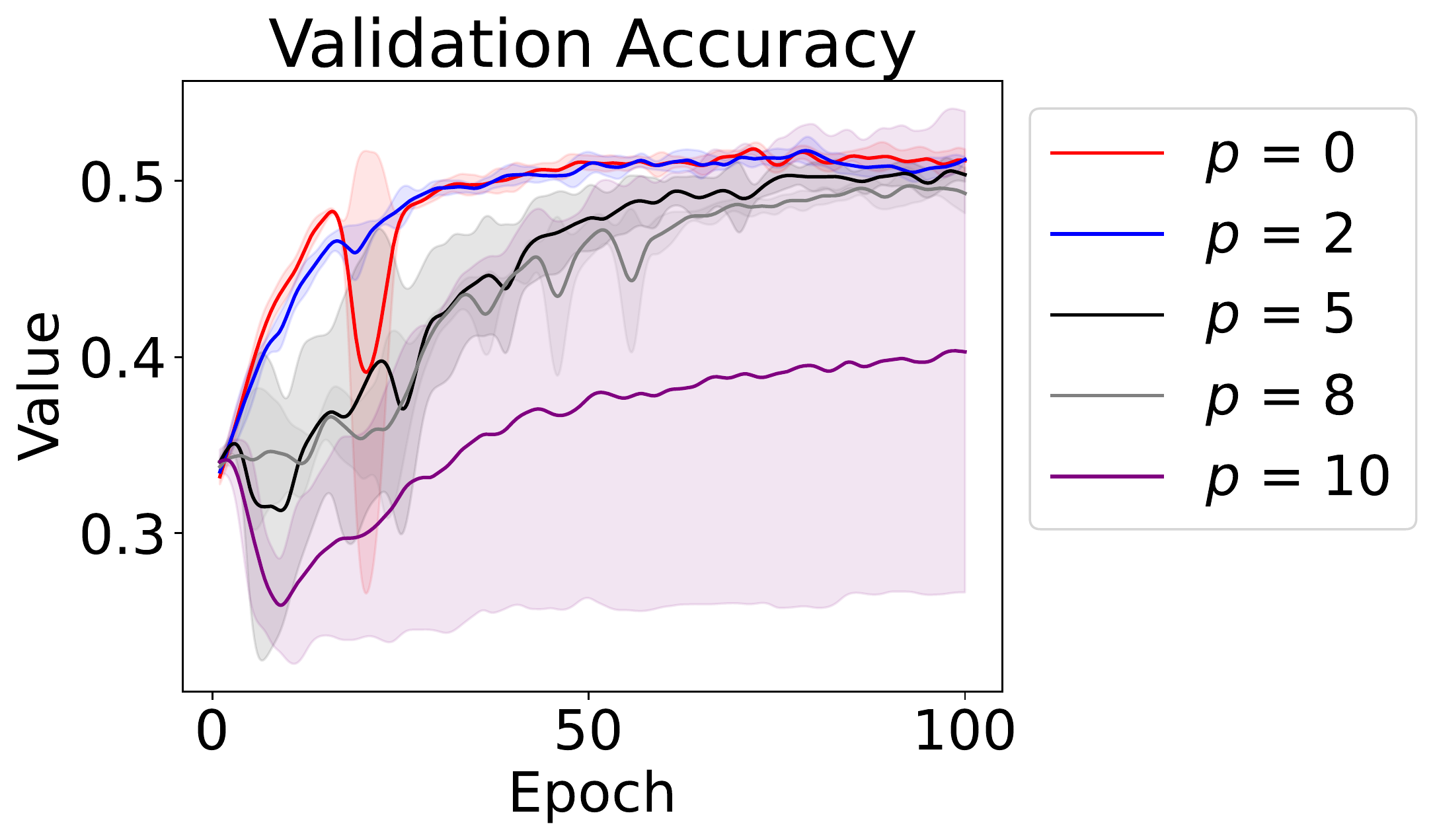}%old/results/newnov21/mininov16Val5W51.0SAccuracy.pdf}
  \caption{Validation}
  \label{fig:appendix_val_eta_trend}
\end{subfigure}
\begin{subfigure}{0.24\textwidth}
  \centering
  \includegraphics[width=\linewidth]{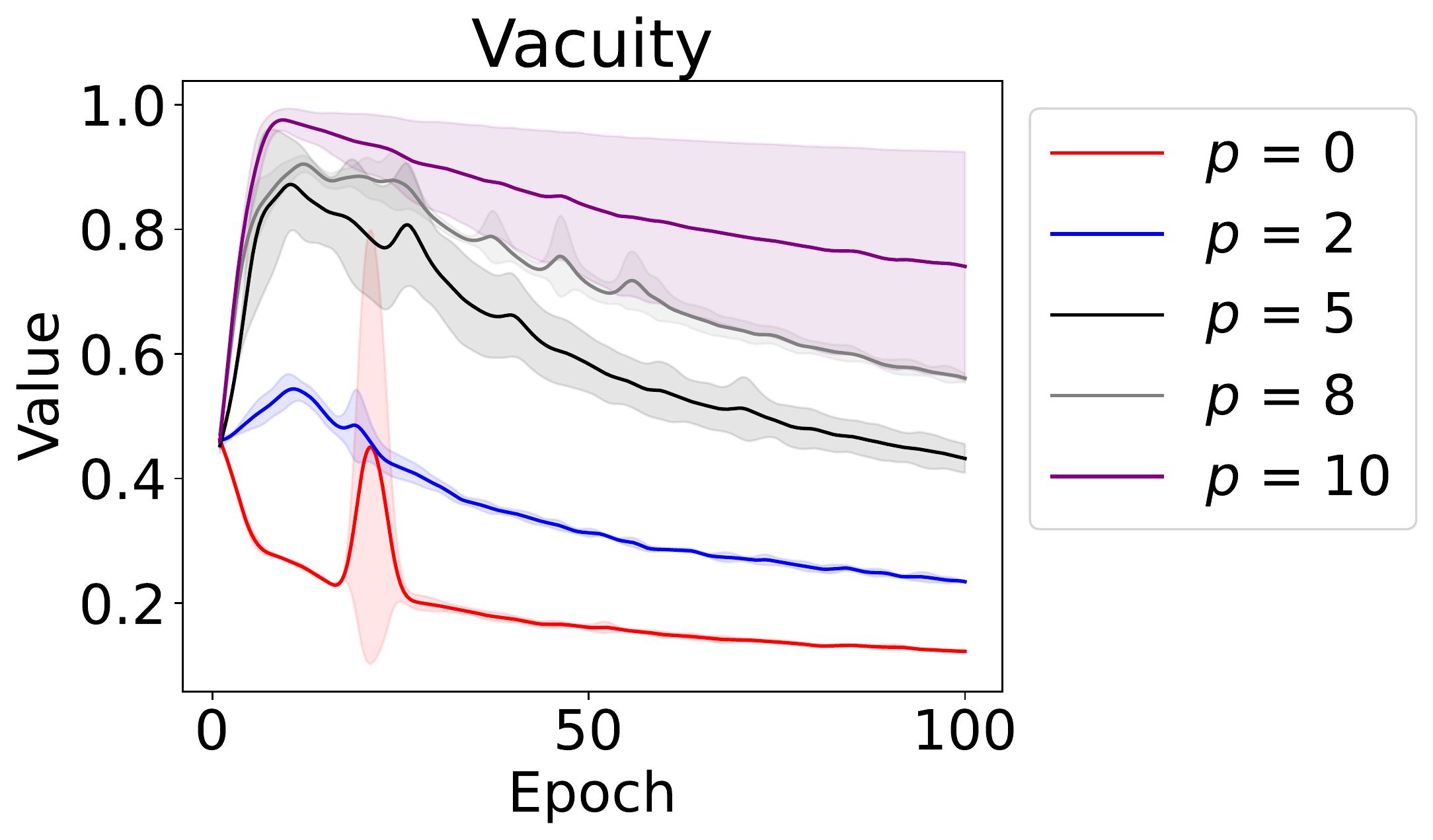}%old/results/newnov21/mininov16Tr5W5S2.0Vacuity.pdf}
  \caption{Vacuity}
\end{subfigure}
\begin{subfigure}{0.24\textwidth}
  \centering
  \includegraphics[width=\linewidth]{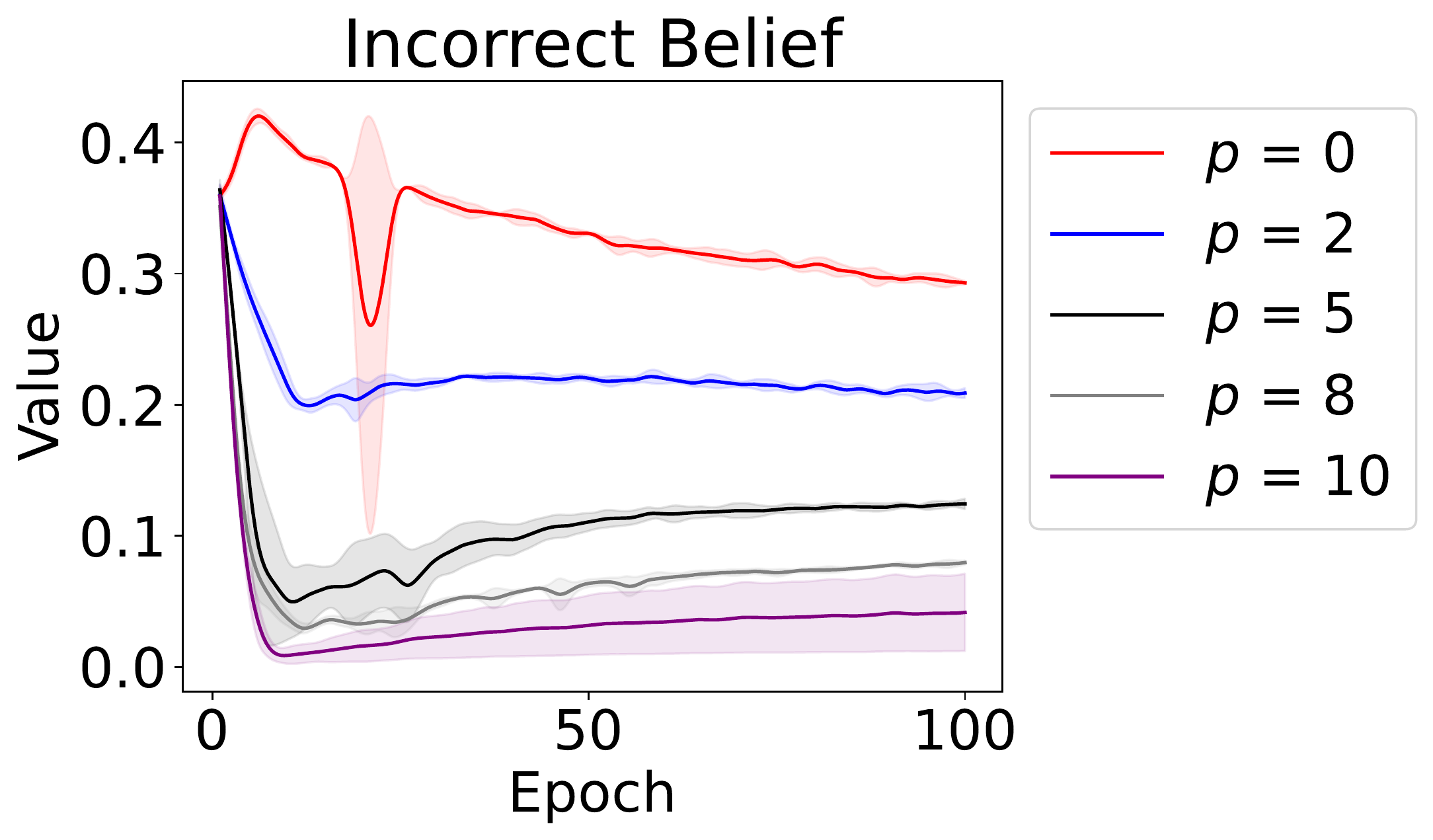}%old/results/newnov21/mininov16Tr5W5S2.0Incorrect Belief.pdf}
  \caption{Incorrect Belief}
\end{subfigure}
\caption{
Impact of regularization on (a) Training Accuracy, (b) Validation Accuracy, (c) Training Vacuity, and (d)Training Incorrect Belief for 5-way 1-shot CifarFS experiment}
\label{fig:appendix_trend_eta_all}
\end{figure}

\begin{wrapfigure}{r}{0.25\textwidth}
\vspace{-10mm}
  \begin{center}
    \hspace{-0.5cm}\includegraphics[width=\linewidth]{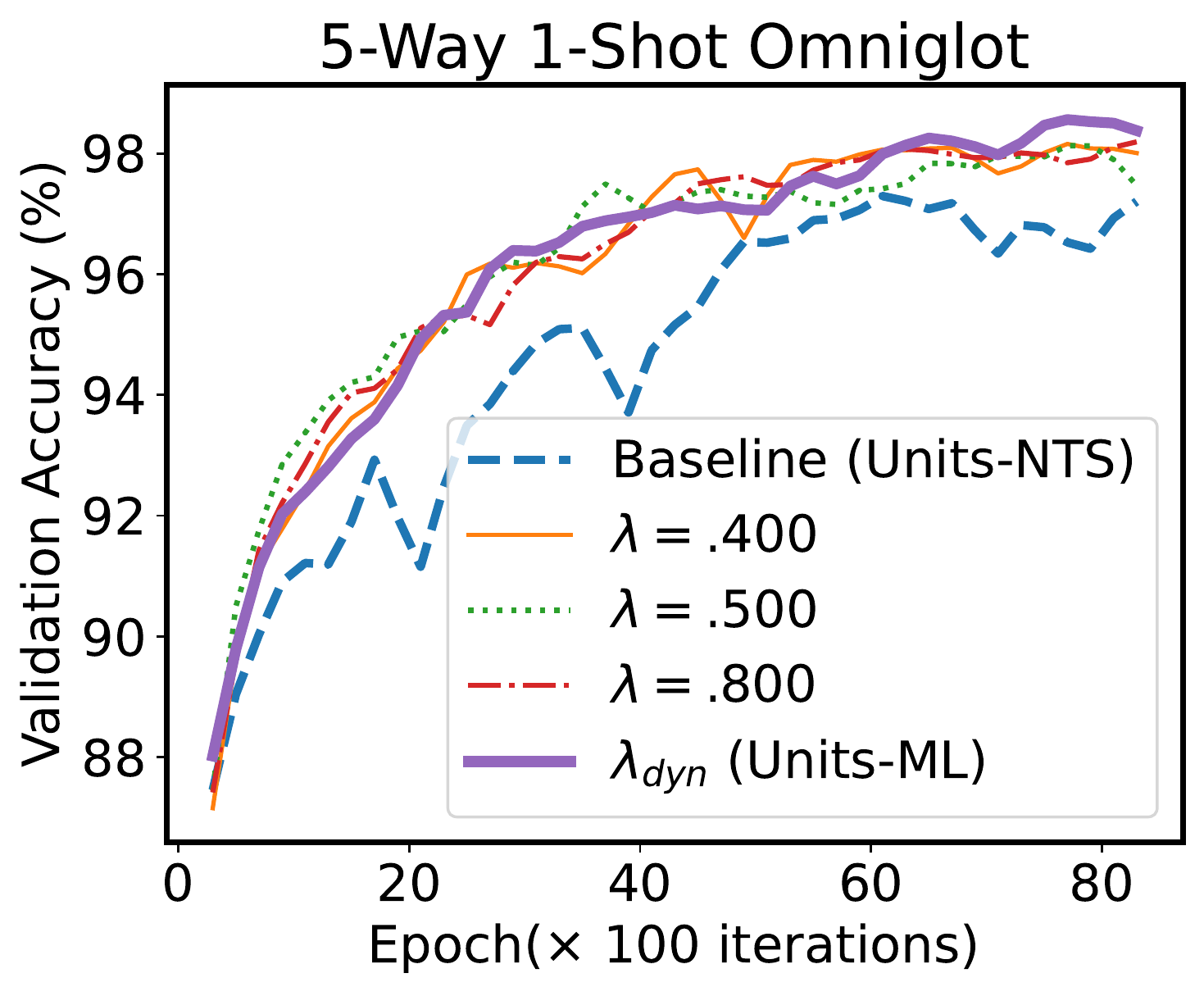}
  \end{center}
  \vspace{-6mm}
  \caption{Impact of $\lambda$}
  \label{fig:diffLambdaTrends}
\vspace{-9mm}
\end{wrapfigure}
\paragraph{Sensitivity to $\lambda$.} The model should focus on acquiring new knowledge (most vacuous tasks) at the initial phase, and as training progresses, transition to correct its acquired but incorrect knowledge. Thus, we set $\lambda$ heuristically to take a relatively large value and gradually decrease as training progresses. Specifically, in all Units-ML experiments,  we start with balancing term $\lambda$ value of $\lambda_{start}=0.99$ and dynamically adjust it  as $\lambda = \lambda_{start} - (\lambda_{start} - \lambda_{end}) \times \min(1.0, E/50)$ as training progresses to reach $\lambda_{end} = 0.5$ at the end of training. Since the vacuous belief ($vb^t$) also decreases as the model explores the task space, the performance is quite robust as shown in Figure~\ref{fig:diffLambdaTrends}, where we tested different $\lambda$ values on 5-way 1-shot tasks using the Omniglot dataset. %We leave exploration of optimal $\lambda$ value for task selection as a possible future work.

\end{document}